\documentclass[11pt]{article}

\usepackage[letterpaper,margin=1in]{geometry}
\usepackage[numbers,sort&compress]{natbib}
\usepackage[T1]{fontenc}
\usepackage[utf8]{inputenc}
\usepackage{amsmath,amssymb,amsthm}
\IfFileExists{newtxtext.sty}
  {\usepackage{newtxtext,newtxmath}}
  {\usepackage{mathptmx}}
\usepackage{microtype}

\usepackage{graphicx}
\usepackage{booktabs}
\usepackage[font=small,labelfont=bf]{caption}
\usepackage{xcolor}

\usepackage{algorithm}
\usepackage{algpseudocode}

\usepackage{tikz}
\usetikzlibrary{automata, positioning, arrows.meta, calc, bending}

\usepackage{url}
\usepackage[colorlinks=true, linkcolor=blue, citecolor=blue, urlcolor=blue]{hyperref}
\usepackage{cleveref}

\IfFileExists{orcidlink.sty}
  {\usepackage{orcidlink}}
  {\newcommand{\orcidlink}[1]{\href{https://orcid.org/#1}{\textsuperscript{iD}}}}

\newcommand{\Description}[1]{}

\tikzset{
  dfa/.style={
    >={Stealth[round]},
    shorten >=1pt,
    on grid,
    auto,
    node distance=3.0cm,
    every state/.style={
      minimum size=1.05cm, inner sep=1pt, semithick,
      draw=black!80, fill=white
    },
    accepting/.style={double, double distance=1.1pt},
  },
  viol/.style={fill=black!8, draw=black!55}, 
}

\newcommand{\phiD}{\varphi_D}
\newcommand{\phiA}{\varphi_A}
\newcommand{\AD}{\mathcal{A}_D}
\newcommand{\AAtk}{\mathcal{A}_A}
\newcommand{\Sh}{\mathrm{Sh}}
\newcommand{\turnD}{\textsc{def}}
\newcommand{\turnA}{\textsc{adv}}
\newcommand{\calH}{\mathcal{H}}   
\newcommand{\Links}{\mathcal{L}}  
\newcommand{\St}{\Sigma_{\calH}}  
\newcommand{\Net}{\mathcal{N}}                  
\newcommand{\Mnet}{\mathcal{M}_{\mathrm{net}}}   
\newcommand{\calW}{\mathcal{W}}
\newcommand{\Attr}{\mathrm{Attr}}
\newcommand{\Safe}{\mathrm{Safe}}
\newcommand{\Legal}{\Lambda}
\newcommand{\Succ}{\mathrm{Succ}}

\theoremstyle{plain}
\newtheorem{theorem}{Theorem}[section]

\usepackage{thmtools}
\declaretheorem[style=definition,sibling=theorem,name=Definition,refname={Definition,Definitions},Refname={Definition,Definitions}]{definition}
\declaretheorem[style=definition,sibling=theorem,name=Remark,refname={Remark,Remarks},Refname={Remark,Remarks}]{remark}
\declaretheorem[sibling=theorem,name=Proposition,refname={Proposition,Propositions},Refname={Proposition,Propositions}]{proposition}

\begin{document}

\title{\bfseries Shielded Analysis:\\[2pt]
Certification and Characterization of Defensibility in Systems under Adversarial Interaction}

\author{%
  Achraf Hsain\,\orcidlink{0009-0000-5817-8109}\thanks{Corresponding author.}\\[2pt]
  \normalsize Information and Computer Science Department\\
  \normalsize King Fahd University of Petroleum and Minerals\\
  \normalsize Dhahran, Saudi Arabia\\
  \normalsize \texttt{g202518970@kfupm.edu.sa}
  \and
  Sultan Almuhammadi\,\orcidlink{0000-0003-0269-1792}\\[2pt]
  \normalsize Information and Computer Science Department\\
  \normalsize King Fahd University of Petroleum and Minerals\\
  \normalsize Dhahran, Saudi Arabia\\
  \normalsize \texttt{muhamadi@kfupm.edu.sa}
}

\date{}

\maketitle

%
%

\begin{abstract}
\setlength{\parindent}{0pt}
\setlength{\parskip}{0.8em}
\noindent
Formal safety analysis determines whether a system admits a safe defense;
adaptive evaluation characterizes the operating quality sustained under
adversarial interaction. Both answers matter because systems with the same
safety verdict can impose very different operational burdens.

We introduce \emph{shielded analysis}, a design-time framework that derives
these answers from one encoded system while keeping the safety requirement and
admissible threat model independently variable. It returns a defensibility
certificate and a four-axis defensibility fingerprint spanning structural
margin, shield latitude, and adaptive operating quality. Each axis is
informative in its own right; their relationships show whether formal and
operational assessments agree, diverge, or respond differently to system
changes.

We instantiate the framework for network defense on a reference segment and
four controlled perturbations spanning topology, safety requirements, and
adversary capabilities. Every configuration is certified defensible, yet two
topology variants with nearly identical structural profiles sustain mean
clean-host fractions of $22.7\%$ and $80.7\%$ under adaptive pressure.

Shielded analysis turns a safety-game solution into a comparative instrument:
it determines whether a defense exists, characterizes what that defense
requires, and identifies which system changes strengthen it.
\end{abstract}


\section{Introduction}
\label{sec:intro}

Shielded reinforcement learning (ShRL)~\cite{alshiekh2018,konighofer2017,bloem2015}
compiles a temporal-logic safety specification into a deterministic finite
automaton, builds a product game with a model of the system, computes the
winning region by attractor fixed-point iteration, and extracts a shield: a
lookup table of permitted actions under which the specification is never
violated.
The literature deploys it as a \emph{runtime enforcement mechanism}: the shield
filters an RL agent's actions during training or deployment, keeping the agent
safe while it optimizes a reward
signal~\cite{alshiekh2018,konighofer2017,konighofer2025}.

The winning region and its fixed-point structure contain design information
beyond the extracted action filter. \emph{Shielded analysis} elevates that solved
game to the primary analytical object and combines its certificate, region
geometry, attractor structure, action latitude, and shield-constrained adaptive
behavior in a unified design-time analysis. Given a system model, defender
requirement, and adversary constraint, it determines whether the configuration
admits a safe defense and characterizes the structural and operational
conditions under which that defense is sustained. This paper formalizes the
framework and validates it through a controlled network-security instantiation.

The framework's object is a two-player turn-based safety game~\cite{gradel2002},
augmented with a second temporal
specification that constrains the adversary and is enforced asymmetrically
against the first: a defender specification defines what must never happen and
seeds the violation set; an adversary specification defines which adversary
behaviors are admissible and constrains the adversary during solution. We call
the augmented object an \emph{asymmetrically constrained safety game}. The
framework returns a \emph{defensibility certificate}: from a
designated initial state, the system either admits a defender strategy
satisfying its specification against every admissible adversary behavior, or
provably admits none. The winning region witnessing that certificate then
bounds adversarial multi-agent reinforcement learning through its derived
shield, adding an operational layer, and metrics over the formal structure and
the learned behavior compose a \emph{defensibility fingerprint}.

We instantiate the framework on network topological safety under adversarial
attack: whether a given network segment is defensible, how defensible it is,
and which architectural changes would improve the margin. Current tools answer these
questions in part. Static configuration
verification~\cite{fogel2015,headerspace} decides reachability and isolation
properties but does not model a strategic adversary.
Reinforcement-learning-based adversarial evaluation~\cite{standen2021,kiely2025}
produces operational readings but no certificate: a policy that survives one
adversary may fail against another, and no amount of training establishes that
a configuration is indefensible. The instantiation returns a certificate and
fingerprint for a reference topology and four controlled perturbations.

We develop this analytical paradigm through three contributions:

\begin{enumerate}
    \item \textbf{Shielded analysis through asymmetrically constrained safety games.}
    We introduce a design-time analytical framework in which two independently
    specified temporal requirements enter distinct stages of game solution:
    $\phiD$ seeds the attractor's initial unsafe set, while $\phiA$ filters the
    adversary's successor set during fixed-point iteration. This asymmetric
    construction keeps the safety requirement and threat boundary independently
    manipulable and returns a defensibility certificate witnessed by the winning
    region and executable as a shield (\Cref{sec:sa-framework}).

    \item \textbf{A relational reading of the solved game.} The defensibility
    fingerprint combines three static readings of the fixed point and shield
    with one adaptive reading of shield-constrained play. Each reading
    characterizes a distinct aspect of defensibility, while their relationships
    support direct comparison of structural and operational assessments
    (\Cref{sec:marl}, \Cref{sec:metrics}).

    \item \textbf{Network-security instantiation and controlled evaluation.}
    We instantiate the framework on a network defense game and evaluate a
    reference configuration under four controlled topology and specification
    changes. The evaluation distinguishes certification from operational
    quality, localizes structural effects by attractor depth, and quantifies how
    system changes affect the four defensibility axes
    (\Cref{sec:instantiation}, \Cref{sec:experiments}).
\end{enumerate}

Our implementation, experiments, and topology data are publicly available.\footnote{\url{https://github.com/AchrafHsain7/Bastion}} Code is released under the MIT License; experimental data and outputs under CC-BY-4.0.

\section{Preliminaries}
\label{sec:prelim}

\Cref{tab:notation} collects the notation used throughout.

\begin{table}[!t]
\centering
\caption{Notation used throughout the paper. Framework-level symbols appear
first, followed by those introduced by the network instantiation
(\Cref{sec:instantiation}).}
\label{tab:notation}
\footnotesize
\begin{tabular}{@{}ll@{}}
\toprule
\textbf{Symbol} & \textbf{Meaning} \\
\midrule
\multicolumn{2}{@{}l}{\textit{Automata and specifications}}\\
$\mathcal{A}=(Q,\Sigma,\delta,q_0,F)$ & Deterministic finite automaton \\
$\square\,\psi$              & LTL ``always $\psi$'' (safety) \\
$\otimes$                    & Synchronous product of automata \\
$\phiD,\ \phiA$              & Defender / adversary specifications \\
$\AD,\ \AAtk$                & The safety automata they compile to \\
$Q_D,Q_A$;\ $F_D,F_A$        & Their state sets; their accepting sets \\
$\delta_D,\delta_A$;\ $\Sigma_D,\Sigma_A$ & Their transition functions; their alphabets \\
$q_D,\ q_A$                  & State of $\AD$ / $\AAtk$ reached along a play \\
$\sigma_D(v,a),\ \sigma_A(v,a)$ & Symbols $\AD$ / $\AAtk$ read on move $a$ from $v$ \\
\addlinespace
\multicolumn{2}{@{}l}{\textit{Game structure and product game}}\\
$\mathcal{C}=(\mathcal{M},\phiD,\phiA)$ & Asymmetrically constrained safety game \\
$\mathcal{M}$                & Game structure (\Cref{def:casg}) \\
$S$;\ $S_D,\ S_A$            & Its states; defender-turn / adversary-turn blocks \\
$s,\ s_0$                    & A state of $\mathcal{M}$; the initial state \\
$\mathrm{Act}_D,\ \mathrm{Act}_A$ & Defender / adversary action sets \\
$\delta$                     & Transition function of $\mathcal{M}$, alternating the turn \\
$\mu: S\to[0,1]$             & Dominance function \\
$G_\times=(V,V_D,V_A,E,v_0)$ & Product game: positions $V = S\times Q_D\times Q_A$, \\
                             & \quad turn partition $V_D,V_A$, edges $E$ \\
$v,\ v_0$                    & A product position; the initial position \\
$U$                          & Initial unsafe set $\Attr_0 = \{v : q_D\notin F_D\}$ \\
$E(v),\ \Legal(v)$           & Successors of $v$; those admissible under $\phiA$ \\
$\Succ(v,a)$                 & Successor of $v$ under action $a$ \\
\addlinespace
\multicolumn{2}{@{}l}{\textit{Safety-game solution}}\\
$\calW,\ \Attr^{*}$          & Winning region; attractor $V\setminus\calW$ \\
$\operatorname{rank}(v)$     & Moves within which violation can be forced from $v$ \\
$\Attr_k,\ \Sh_k$            & Attractor after iteration $k$; shell $\Attr_k\setminus\Attr_{k-1}$ \\
$k_{\max},\ p_k,\ H$         & Deepest shell; shell mass fraction; entropy of $(p_k)$ \\
$\Safe(v),\ \mathcal{R}$     & Shield-permitted actions at $v$; shielded-reachable set \\
\addlinespace
\multicolumn{2}{@{}l}{\textit{Defensibility metrics}}\\
$\mathrm{WIN},\ \mathrm{STP}$ & Winning fraction; shell steepness (static) \\
$\mathrm{SLT},\ \mathrm{DDR}$ & Shield latitude (static); defender dominance (adaptive) \\
\addlinespace
\multicolumn{2}{@{}l}{\textit{Learning}}\\
$o(v),\ Q_i(o(v),a)$         & Learner observation $o(v)=s$; Q-value, $i\in\{D,A\}$ \\
$\alpha,\ \gamma,\ R_i$      & Learning rate; discount; per-player reward \\
$\varepsilon,\ N_{\mathrm{ep}},\ T$ & Exploration rate; training episodes; steps per episode \\
$N_{\mathrm{run}},\ N_{\mathrm{rep}}$ & Independent training runs; reporting-window length \\
$S_{\mathrm{init}},\ p$      & Re-initialization set; engagement-reset probability \\
\midrule
\multicolumn{2}{@{}l}{\textit{Network instantiation}}\\
$\Net=(\calH,\Links)$        & Network segment: hosts $\calH$, directed links $\Links$ \\
$h$                          & A host, $h\in\calH$ \\
$\St=\{C,X,D,I,Z\}$          & Host statuses (Clean, Compromised, Detected, Isolated, Destroyed) \\
$\Mnet$                      & Network defense game: the $\mathcal{M}$ of this instance (\Cref{def:ndg}) \\
$\tau \in \{\turnD,\turnA\}$ & Turn indicator: defender / adversary to move \\
$n_C(s),\ \mathrm{active}(s)$ & Clean hosts; hosts neither Isolated nor Destroyed \\
\bottomrule
\end{tabular}
\end{table}

\begin{definition}[Deterministic Finite Automaton \cite{baier2008}]
A DFA is a quintuple $\mathcal{A} = (Q, \Sigma, \delta, q_0, F)$ with finite state
set $Q$, finite input alphabet $\Sigma$, transition function
$\delta: Q \times \Sigma \to Q$, initial state $q_0$, and set of accepting states
$F \subseteq Q$.
\end{definition}

\begin{definition}[Safety Automaton \cite{baier2008}]
A \emph{safety automaton} is a DFA used to recognize a property the system must
\emph{always} satisfy: a run is safe if and only if every state it visits belongs
to $F$. We assume without loss of generality that the violating states are
\emph{absorbing} --- for every $q \notin F$ and every $\sigma \in \Sigma$,
$\delta(q,\sigma) \notin F$ --- so that once a run leaves $F$ it can never return
and a safety violation is permanent.
\end{definition}

Linear temporal logic (LTL) safety formulas of the form $\square\, \psi$
(``always $\psi$'') compile mechanically into safety automata~\cite{baier2008,kupferman2001}.
When several safety properties must hold simultaneously, their automata are
combined via \emph{synchronous product}. Given DFAs
$\mathcal{A}_1 = (Q_1, \Sigma_1, \delta_1, q_{0}, F_1)$ and
$\mathcal{A}_2 = (Q_2, \Sigma_2, \delta_2, q^{'}_{0}, F_2)$

\[
  \mathcal{A}_1 \otimes \mathcal{A}_2 =
  \big(Q_1 \times Q_2,\ \Sigma_1 \times \Sigma_2,\ \delta_\otimes,\ (q_{0}, q^{'}_{0}),\ F_1 \times F_2\big),
  \quad
  \delta_\otimes\big((q_1,q_2),(\sigma_1,\sigma_2)\big) = \big(\delta_1(q_1,\sigma_1),\, \delta_2(q_2,\sigma_2)\big).
\]
Each component reads its own alphabet, and the product accepts exactly when both
components do. Where the two automata read a common alphabet this is
$L(\mathcal{A}_1) \cap L(\mathcal{A}_2)$; either way the product enforces both
properties at once.

\begin{definition}[Safety Game \cite{gradel2002}]
A \emph{safety game} is played on a game arena
$\Gamma = (V, V_D, V_A, E, v_0)$, where $(V, E)$ is a directed graph,
$(V_D, V_A)$ is a partition of $V$ called the \emph{turn partition}, and
$v_0 \in V$ is the initial position. We call the vertices \emph{positions} and the
edges \emph{moves}, and we assume every position has at least one outgoing move so
that plays are infinite. Player~0 (the \emph{defender}) owns the positions in
$V_D$; Player~1 (the \emph{adversary}) owns those in $V_A$. The owner of the current
position chooses an outgoing move. A \emph{safety objective} is given by an unsafe
set $U \subseteq V$: an infinite play is winning for Player~0 if and only if it
never visits $U$.
\end{definition}

\begin{definition}[Winning Region and Attractor \cite{gradel2002}]
The \emph{winning region} $\calW \subseteq V$ is the largest set of positions from
which Player~0 has a strategy guaranteeing the safety objective: from every
$v \in \calW$, Player~0 can ensure $U$ is never visited, regardless of Player~1's
strategy. Its complement, the \emph{attractor} $\Attr^* = V \setminus \calW$, is the
set of positions from which Player~1 can force a visit to $U$ in finitely many
moves, no matter how Player~0 plays.
\end{definition}

\paragraph{Attractor Computation.}
The attractor is computed as a least fixed point. Initialize $\Attr_0 = U$. At each
iteration $i$: (1)~add any Player~1 position $v \notin \Attr_i$ that has \emph{at
least one} successor in $\Attr_i$ (the adversary needs only one path to danger);
(2)~add any Player~0 position $v \notin \Attr_i$ whose \emph{every} successor lies
in $\Attr_i$ (the defender is lost only when every successor is already in the
attractor). The process terminates at the least fixed point $\Attr^*$ when no new positions are added.
Complexity is $O(|V| + |E|)$~\cite{gradel2002,baier2008}. The positions added at
iteration $k$ form what we call the \emph{attractor shell}
$\Sh_k = \Attr_k \setminus \Attr_{k-1}$: positions exactly $k$ optimal moves from
violation. These shells carry structural information about the game graph that we
exploit for metric construction (\Cref{sec:metrics}).

\paragraph{Minimax Q-Learning.}
In a two-player zero-sum Markov game~\cite{littman1994}, the two players' rewards
sum to zero. Where both players act simultaneously, minimax Q-learning evaluates
a successor by solving a matrix game at each state for the minimax value over
mixed strategies, and convergence of that update to the Markov-game value is
established under standard conditions~\cite{littman1996,littman2001}. The game
of \Cref{def:casg} is turn-based: exactly one player chooses at each position.
There is no simultaneous move to randomize over, and finite discounted
turn-based games admit optimal pure positional strategies, so the Bellman
operation at a position is a maximum or a minimum according to ownership. Under
a current-player value convention this is the negated opponent maximum, and we
maintain a Q-table per player updated by the cross-table rule
\[
  Q_i(s,a) \leftarrow Q_i(s,a) + \alpha\Big[R_i + \gamma \cdot \big(-\max_{a'} Q_j(s',a')\big) - Q_i(s,a)\Big],
\]
where $j$ denotes the opponent. \Cref{sec:marl} instantiates this update under
the shield constraints and declared adaptive protocol.

\section{The Shielded-Analysis Framework}
\label{sec:sa-framework}

This section defines the defensibility certificate and fingerprint, followed by
the game construction, solution algorithm, shield extraction, and two-layer
methodology.

\subsection{Asymmetrically Constrained Safety Games}
\label{sec:casg}

Two objects carry the framework. The analyst supplies $\mathcal{C}$: a system
model together with two specifications. From $\mathcal{C}$ the framework
derives $G_\times$, the arena the solver runs on. Every result that
follows---certificate, winning region, attractor shells, shield---is a
property of $G_\times$.

\begin{definition}[Asymmetrically Constrained Safety Game]
\label{def:casg}
An \emph{asymmetrically constrained safety game} is a tuple
$\mathcal{C} = (\mathcal{M}, \phiD, \phiA)$ with the following components.
\begin{itemize}
  \item $\mathcal{M} = (S, S_D, S_A, s_0, \mathrm{Act}_D, \mathrm{Act}_A, \delta, \mu)$
  is a two-player turn-based deterministic game
  structure~\cite{gradel2002,piterman2006} augmented with a dominance
  function: a finite state space $S$ partitioned into defender-turn states
  $S_D$ and adversary-turn states $S_A$; an initial state $s_0$; finite
  defender and adversary action sets $\mathrm{Act}_D$ and $\mathrm{Act}_A$; a
  deterministic transition function $\delta$ alternating the turn; and a
  \emph{dominance function} $\mu : S \to [0,1]$ scoring how favorable a state
  is to the defender.

  \item $\phiD$ is a \emph{defender specification}: a temporal safety property
  over the states of $\mathcal{M}$, compiled into a safety automaton
  $\AD = (Q_D, \Sigma_D, \delta_D, q_{0,D}, F_D)$.

  \item $\phiA$ is an \emph{adversary specification}: a temporal safety
  property over the adversary's behavior, compiled into a safety automaton
  $\AAtk = (Q_A, \Sigma_A, \delta_A, q_{0,A}, F_A)$.
\end{itemize}
\end{definition}

The dominance function is inert for the solution: the certificate, the winning
region, and the shell decomposition are computed from
$(S, S_D, S_A, s_0, \mathrm{Act}_D, \mathrm{Act}_A, \delta)$ alone. Only the
adaptive metric of \Cref{sec:ddr} reads $\mu$.

The two automata run synchronously with the game on a product state
$(s, q_D, q_A)$, and enter the solution at different points.

$\phiD$ is the safety objective. Either player may drive $\AD$ out of
$F_D$---the adversary by attacking, the defender by responding badly---and the
product states with $q_D \notin F_D$ form the unsafe set from which the
solution starts.

$\phiA$ delimits the adversary's admissible behavior: a resource bound, an
ordering constraint, a capability restriction. Adversary moves that would take
$\AAtk$ out of $F_A$ are removed during solution. Since $\phiA$ is a property of what the adversary does, only adversary moves
take $\AAtk$ out of $F_A$. Defender moves may advance it, as a constraint
referring to the defender's actions requires, but never out of its accepting
set: we require of every admissible $\phiA$ that
\begin{equation}
\label{eq:phia-defender}
q_A \in F_A,\ v \in V_D
\;\Longrightarrow\;
\delta_A\bigl(q_A, \sigma_A(v,a)\bigr) \in F_A
\qquad\text{for every defender action } a .
\end{equation}
Without \eqref{eq:phia-defender} a defender move could drive $\AAtk$ into its
rejecting sink, and the rejecting region would no longer be reachable only
through the edges $\Legal$ filters.

$\phiD$ thus fixes the target of the fixed-point computation and $\phiA$ its
edge set. Neither alphabet is constrained: both automata may read state labels,
actions, or both.

\begin{definition}[Product Game]
\label{def:product}
Let $\sigma_D(v,a) \in \Sigma_D$ and $\sigma_A(v,a) \in \Sigma_A$ be the
symbols $\phiD$ and $\phiA$ read on move $a$ from $v$. The \emph{product game}
of $\mathcal{C}$ is the arena
$G_\times = (V, V_D, V_A, E, v_0)$ with
$V = S \times Q_D \times Q_A$,
$V_D = S_D \times Q_D \times Q_A$,
$V_A = S_A \times Q_D \times Q_A$, initial position
\[
  v_0 = \bigl(s_0,\ \delta_D(q_{0,D}, \sigma_D(s_0)),\ \delta_A(q_{0,A}, \sigma_A(s_0))\bigr),
\]
so that both automata consume the label of $s_0$ when the product game is
built---a system whose initial state already violates $\phiD$ is then seeded
directly into $U$ rather than absorbed one iteration later---and an edge
$\bigl((s,q_D,q_A),(s',q_D',q_A')\bigr) \in E$
whenever, for some action $a$ of the player owning $s$,
\[
  s' = \delta(s,a), \qquad
  q_D' = \delta_D\bigl(q_D, \sigma_D(v,a)\bigr), \qquad
  q_A' = \delta_A\bigl(q_A, \sigma_A(v,a)\bigr).
\]
Because $\delta$ and both automata are deterministic, each action of the player
owning $v$ induces exactly one such $v'$; we write $\Succ(v,a) = v'$ for it.
Its \emph{unsafe set} is $U = \{v \in V : q_D \notin F_D\}$, and its
\emph{admissible-successor map} is
\[
  \Legal(v) =
  \begin{cases}
    \{v' : (v,v') \in E\} & v \in V_D,\\[2pt]
    \{v' : (v,v') \in E,\ q_A' \in F_A\} & v \in V_A.
  \end{cases}
\]
\end{definition}

\begin{remark}[Scope]
\label{rem:scope}
The framework as developed here is explicit-state, turn-based, and
deterministic: $\delta$ is a function, not a distribution, and both players
observe the product state fully. \Cref{sec:discussion} examines what these
assumptions imply for computational scope and extension.
\end{remark}

\subsection{The Defensibility Certificate}
\label{sec:certificate}

\begin{definition}[Defensibility Certificate]
\label{def:defensibility}
Let $\mathcal{C}$ be an asymmetrically constrained safety game with product
game $G_\times$ and winning region $\calW$. $\mathcal{C}$ is
\emph{defensible} iff $v_0 \in \calW$. The \emph{defensibility certificate}
for $\mathcal{C}$ is exactly one of (i)~$v_0 \in \calW$, and $\calW$ witnesses
a defender strategy satisfying $\phiD$ against every $\phiA$-admissible
adversary behavior, or (ii)~$v_0 \notin \calW$, and provably no such strategy
exists.
\end{definition}

The dichotomy is exhaustive because finite safety games are determined: from
every position, exactly one of the two players has a winning
strategy~\cite{gradel2002}. Determinacy is what licenses the negative
branch---indefensibility is a theorem about the game, not a failure to find a
strategy.

The shield---the table of safe actions per winning state---is a derived
artifact: a witness to the certificate's positive case
(\Cref{sec:extraction}). When positive, $\calW$ identifies exactly which
configurations admit a safe defense strategy. When negative, the framework
certifies that no defense exists from the initial configuration.

\subsection{Asymmetric Enforcement}
\label{sec:asym-general}

The two specifications serve different roles in the solution:

\begin{itemize}
    \item \textbf{$\phiD$ defines the safety objective.} A play is safe iff
    $\AD$ remains in an accepting state at every step. Product states with
    $q_D \notin F_D$ seed the attractor's initial unsafe set $U$.
    \item \textbf{$\phiA$ constrains the adversary's behavior.} The
    adversary's moves are filtered to those keeping $\AAtk$ in an accepting
    state, through $\Legal$ (\Cref{def:product}), during attractor
    computation.
\end{itemize}

Seeding adversary-violation states into $U$ is incorrect wherever the adversary
can eventually force $\AAtk$ out of $F_A$:
every adversary position then becomes a predecessor of $U$ and the attractor
absorbs the state space. The forcing need not be available in one step;
\Cref{sec:dual_spec} works the argument through for a consumable budget. The construction adopted here uses $\phiD$ to define what
must never happen and $\phiA$ to define what the adversary cannot do.
\begin{proposition}[Monotonicity of the winning region under adversary restriction]
\label{prop:refinement}
Let $E(v)$ denote the successor set of $v$ in $G_\times$ and let $\Legal(v)$ be
the admissible-successor map of \Cref{def:product}, so that
$\Legal(v) = E(v)$ for every $v \in V_D$ and
$\Legal(v) \subseteq E(v)$ for every $v \in V_A$. Let $\Attr^{*}_{E}$ be the
adversary attractor over $G_\times$ computed with successor map $E$, and let
$\Attr^{*}$ be the one computed with $\Legal$, as everywhere else in this
paper; write $\calW_{E} = V \setminus \Attr^{*}_{E}$ and
$\calW = V \setminus \Attr^{*}$. Then
\[
  \Attr^{*} \subseteq \Attr^{*}_{E}, \qquad\text{hence}\qquad
  \calW_{E} \subseteq \calW.
\]
Constraining the adversary can only enlarge or preserve the defender's winning
region; it can never shrink it.
\end{proposition}

\begin{proof}
By induction on the iteration index $k$. At $k = 0$ both iterates equal $U$.
Assume $\Attr_{k} \subseteq \Attr_{E,k}$ and let
$v \in \Attr_{k+1}$. If $v \in V_A$, then some
$v' \in \Legal(v) \subseteq E(v)$ lies in
$\Attr_{k} \subseteq \Attr_{E,k}$, so $v \in \Attr_{E,k+1}$. If $v \in V_D$,
then $E(v)$ is identical in both games and every successor lies in
$\Attr_{k} \subseteq \Attr_{E,k}$, so again $v \in \Attr_{E,k+1}$. The
inclusion therefore holds at every $k$, and finiteness of $V$ gives
convergence of both iterations. Complementing yields
$\calW_{E} \subseteq \calW$.
\end{proof}

\begin{remark}[Direction of the asymmetry]
\label{rem:refinement}
The result depends on restricting \emph{only} adversary edges. Removing a
defender edge would weaken the universal condition at defender vertices and
could \emph{enlarge} the attractor, shrinking the winning region.
\end{remark}

\paragraph{Comparison with related architectures.}
Reactive synthesis with environment assumptions~\cite{chatterjee2008,piterman2006,bloem2012}
wires environment behavior into the synthesis objective as an implication:
\emph{env satisfies $A$ $\to$ sys satisfies $G$}. The two specifications
collapse into a single synthesis target. Assumptions in that setting may
themselves be derived by pruning environment edges~\cite{chatterjee2008}, but
the assumption is discharged inside one synthesis objective rather than
standing as a second, externally supplied specification that filters the
adversary's successors at every step of the fixed point. Multi-agent shielding~\cite{elsayedaly2021,xiao2023} shields
cooperative agents under a shared specification. The architecture introduced
here makes the two specifications independently manipulable analytical
variables. $\phiD$ seeds the unsafe set; $\phiA$ filters the adversary's
successors during attractor computation. Each automaton therefore controls a
different stage of the fixed-point iteration, supporting certificate and
sensitivity analyses with independently varied safety requirements and threat
boundaries.

\subsection{Solving the Game}
\label{sec:solving}

Algorithm~\ref{alg:attractor} realizes asymmetric adversary enforcement directly
in the predecessor operators: adversary predecessors are evaluated over the
$\phiA$-legal successor relation, while defender predecessors retain the
universal safety condition over the game.

\begin{algorithm}[t]
\caption{Attractor Computation with Constrained Adversary}
\label{alg:attractor}
\begin{algorithmic}[1]
\Require Product game $G_\times$, unsafe set $U = \{v : q_D \notin F_D\}$
\State $i \gets 0$;\quad $\Attr_0 \gets U$
\Repeat
    \State $\Attr_{i+1} \gets \Attr_i$ \Comment{each pass reads $\Attr_i$ only}
    \For{each adversary vertex $v \in V_A \setminus \Attr_i$}
        \If{$\exists\, v' \in \Legal(v) \cap \Attr_i$} \Comment{$\exists$-quantifier}
            \State $\Attr_{i+1} \gets \Attr_{i+1} \cup \{v\}$
        \EndIf
    \EndFor
    \For{each defender vertex $v \in V_D \setminus \Attr_i$}
        \If{$\forall\, v' \in \Legal(v),\ v' \in \Attr_i$} \Comment{$\forall$-quantifier}
            \State $\Attr_{i+1} \gets \Attr_{i+1} \cup \{v\}$
        \EndIf
    \EndFor
    \State $i \gets i + 1$
\Until{$\Attr_i = \Attr_{i-1}$}
\State $\Attr^{*} \gets \Attr_i$
\State \Return $\calW \gets V \setminus \Attr^{*}$
\end{algorithmic}
\end{algorithm}

The quantifiers differ by vertex ownership. An adversary vertex enters the
attractor if \emph{any} admissible successor is already there: the adversary
needs only one path to danger. A defender vertex enters only if \emph{every}
successor is in the attractor: the defender is lost only when every option
leads into the attractor.

The filter can leave $\Legal(v) = \emptyset$ at an adversary vertex all of whose
moves are inadmissible---in the instantiation of \Cref{sec:instantiation}, at
every adversary position whose attacker automaton already sits in its rejecting
sink, which is one step beyond a spent budget: with the budget spent but $\AAtk$
still accepting, every non-\textsc{Destroy} move remains admissible. The
existential test
then fails vacuously and the vertex is winning, which is the intended reading: an
adversary with no $\phiA$-admissible move can force nothing. Defender vertices
retain their full successor set, so the assumption that every position has an
outgoing move is needed only there. Such positions are unreachable from $v_0$,
since $\Legal$ removes the very edges that would enter them.

\subsection{Shield Extraction}
\label{sec:extraction}

For each product state $v \in \calW$, the shield computes the set of safe
actions:
\begin{equation}
\Safe(v) = \begin{cases}
\{a : \Succ(v,a) \notin \Attr^*\} & \text{if } v \in V_D \\[2pt]
\{a : \Succ(v,a) \notin \Attr^* \;\wedge\; \Succ(v,a) \in \Legal(v)\} & \text{if } v \in V_A
\end{cases}
\end{equation}

Since $U \subseteq \Attr^*$, a defender move into a $\phiD$-violating state is
excluded by the attractor test alone. The adversary case carries the second
condition because the existential test at adversary vertices ranges over
$\Legal(v)$: a $\phiA$-violating successor is neither seeded into $U$ nor
examined during iteration, so nothing places it in $\Attr^*$.

Here the shield defines the \emph{arena} within which two adversarial agents
compete. Neither agent
can leave the winning region, and the adversary cannot leave $\AAtk$'s
accepting region. The MARL analysis of \Cref{sec:marl} takes place entirely
within this bounded arena. Every trajectory observed during training therefore
satisfies $\phiD$ throughout, and satisfies $\phiA$ within each engagement
segment; where an instantiation resets engagements (\Cref{sec:respawn}), the
concatenation of segments is read under the renewal interpretation rather than
as one $\phiA$-satisfying play.

\subsection{The Two-Layer Methodology}
\label{sec:two-layer}

The certificate is binary, but defensibility admits degrees. The framework
therefore attaches a second layer to the first. The \emph{static layer}
comprises deterministic functionals of the formal structure: the winning
region, the attractor shell decomposition, and metrics derived from them.
The \emph{adaptive layer} confines two adversarial reinforcement learners to
$\calW$ and measures the operating quality attained under the declared
protocol. The two layers interrogate different
aspects of defensibility---existence of a defense, and the operational quality
attained inside it---and the metrics of \Cref{sec:metrics}, composed across both, form
the defensibility fingerprint.

\section{Shielded Adversarial MARL}
\label{sec:marl}

The solved game is the primary analytical object: $\calW$ witnesses the
certificate, and the derived shield constrains the arena in which two
adversarial learners estimate operational quality.

\subsection{Zero-Sum Reward}

The reward must produce adversarial engagement inside the shield's constraints.
Within $\calW$ no trajectory reaches a $\phiD$-violating state by construction,
so any reward keyed to violation is structurally vacuous and yields no learning
signal. The dominance function $\mu$ of \Cref{def:casg} already scores states on
a defender-favorability scale; rescaling it to a symmetric range gives a
zero-sum objective:
\begin{equation}
R_D(s) = 2\mu(s) - 1, \qquad R_A(s) = -R_D(s), \qquad R_D, R_A \in [-1, +1].
\end{equation}
The defender receives positive reward exactly when $\mu(s) > 1/2$ and
continuous negative reward below it, creating persistent pressure to restore
favorable states rather than accept a degraded operating condition; the adversary is
symmetrically motivated to drive $\mu$ below $1/2$ and hold it there. Because
$R_D$ is an affine image of $\mu$, the reward the learners optimize and the
functional through which their play is later measured are the same object up to
scale.

\subsection{Minimax Q-Learning with Shield Constraints}

The learners play on $G_\times$. We write $v = (s, q_D, q_A)$ for a product state
and $s$ for its $\mathcal{M}$-component, so that $\Safe(v)$ and $\mu(s)$ are both
defined along a play. The two layers read that state at different resolutions.
The shield is evaluated on the full product state, whereas each learner reads
only the \emph{observation}
\begin{equation}
\label{eq:obsmap}
o(v) = s .
\end{equation}
Positions that differ only in automaton state therefore share Q-values but may
admit different actions. The temporal constraints are already compiled into
$\Safe(\cdot)$. Two separate
Q-tables $Q_A$ and $Q_D$ over $o(v)$ are maintained and updated by tabular
Q-learning~\cite{watkins1992,sutton2018} with cross-table minimax lookup:
\begin{align}
Q_A(o(v),a) &\leftarrow Q_A(o(v),a) + \alpha\bigl[R_A + \gamma\bigl(-\max_{a' \in \Safe(v')} Q_D(o(v'),a')\bigr) - Q_A(o(v),a)\bigr] \\
Q_D(o(v),a) &\leftarrow Q_D(o(v),a) + \alpha\bigl[R_D + \gamma\bigl(-\max_{a' \in \Safe(v')} Q_A(o(v'),a')\bigr) - Q_D(o(v),a)\bigr]
\end{align}

The $\Safe(\cdot)$ constraint applies to both the minimax lookahead and to
action selection, and no further filtering is applied at either player.
$\Safe(v)$ already carries the admissibility test at adversary states, where
its definition ranges over $\Legal(v)$ (\Cref{sec:extraction}); at defender
states the attractor test alone suffices, since $\phiD$ violations seed
$\Attr_0$ (\Cref{sec:asym-general}).

\paragraph{The adaptive metric measures play against an adversarially trained opponent.}
The turn-based update of \Cref{sec:prelim} evaluates every successor under the
assumption that the opponent acts to minimize the updating player's value. The
dominance the metric reports is therefore sustained against an adversarially
trained, reward-seeking opponent rather than an average-case or fixed one. This
holds the adaptive layer to the threat model of the static layer: the attractor
computation certifies that a safe defender strategy exists against every
$\phiA$-admissible adversary, and the adaptive metric quantifies the dominance the
defender sustains when that same adversary is also adaptive and reward-seeking,
with both learners confined to $\calW$. The two layers interrogate one adversary
model at two resolutions: whether a safe strategy exists, and the operational
quality of that strategy under adaptive pressure. Both the adversary model and
the arena are inherited from the automata layer.

\subsection{Engagement Resets}
\label{sec:respawn}

An engagement-scoped resource bound becomes exhausted early within a long
training episode. The remainder then measures an adversary that can no longer
act, making the result depend on episode length rather than the specification.

For instantiations in which $\phiA$ is intended to bind per engagement, we adopt
an \emph{engagement reset}: whenever play reaches a designated re-initialization
set $S_{\mathrm{init}} \subseteq S$, then with probability $p$ the adversary
returns to its entry configuration and the automata $\AD$ and $\AAtk$ to their
initial states, beginning a new engagement. Reaching $S_{\mathrm{init}}$ is the
trigger for the reset, not its result. The operation is external to $\delta$: it
is a sampling device over repeated engagements, not a move of $\mathcal{M}$. What
the adaptive layer therefore measures is a renewal process, formed by
concatenating legal engagement segments of the same certified arena, and $\phiA$
is read as a per-engagement constraint---it binds within one engagement, and a
reset begins another.

The renewal protocol is valid when initialized reset targets lie in $\calW$ and
$\phiD$ is preserved across reset boundaries. Both conditions hold for the
instantiation of \Cref{sec:instantiation}, whose safety clauses are predicates
of the current state. History-dependent defender specifications instead require
their automaton memory to persist across resets. The adaptive reading adds this
declared renewal protocol over the arena fixed by the certificate.

\section{Defensibility Metrics}
\label{sec:metrics}

Once a configuration is certified, four quantities grade its defensibility:
how much of the non-violating state space the defender can hold, how the danger
region is shaped, how much of the defender's action repertoire the shield
leaves available, and how well the defense holds when both parties adapt. Three
are functionals of the solved game---the attractor and the shield---and require
no play; we term these the \emph{static} metrics. The fourth is a functional of
the shielded learners of \Cref{sec:marl}, and we term it the \emph{adaptive}
metric. All four take values in $[0,1]$ and are oriented so that larger values
denote a more defensible system. Plotted on common axes they form the
\emph{defensibility fingerprint}.

\subsection{The Attractor Rank Decomposition}
\label{sec:shells}

For $v \in \Attr^{*}$, the \emph{attractor rank} $\operatorname{rank}(v)$ is the
least $n$ such that the adversary can force the play into $U$ within $n$ moves
against every defender strategy, with $\operatorname{rank}(v) = 0$ for
$v \in U$. Evaluated synchronously---each pass of \Cref{alg:attractor} reading
the previous iterate rather than the partially updated one---the $k$-th iterate
is $\Attr_k = \{v : \operatorname{rank}(v) \leq k\}$, so the \emph{shells}
\begin{equation}
\Sh_k = \Attr_k \setminus \Attr_{k-1}
    = \{v \in V : \operatorname{rank}(v) = k\},
\qquad k = 1, \ldots, k_{\max},
\end{equation}
partition the absorbed states by their distance, in forced moves, from
violation. The \emph{depth distribution} is
\begin{equation}
p_k = \frac{|\Sh_k|}{\sum_{j} |\Sh_j|}, \qquad k = 1, \ldots, k_{\max}.
\end{equation}
Because the iterate is read synchronously, the shell index equals the forced
distance to violation and does not depend on the order in which vertices are
visited.

\subsection{Winning Fraction}
\label{sec:win}

\begin{equation}
\mathrm{WIN} = \frac{|\calW|}{|V| - |U|} \;\in\; [0,1]
\end{equation}

$\mathrm{WIN}$ is the fraction of states outside $U$ from which the defender
can guarantee safety. $\mathrm{WIN} = 1$ when
every such state is winnable and $\mathrm{WIN} \to 0$ when the adversary can
force violation from almost everywhere outside $U$.

$\mathrm{WIN}$ is a structural fraction over the full product space, without
reachability or occupancy weighting. Restricting its domain to the
shielded-reachable set $\mathcal{R}$ would make the fraction identically $1$
because $\mathcal{R} \subseteq \calW$. The full-space convention therefore
includes unreachable winning positions, including adversary dead ends.

\subsection{Shell Steepness}
\label{sec:stp}

Let $H = -\sum_{k : |\Sh_k| > 0} p_k \ln p_k$ be the Shannon entropy of the depth
distribution. Shell steepness is the normalized entropy deficit,
\begin{equation}
\mathrm{STP} =
\begin{cases}
1 - \dfrac{H}{\ln k_{\max}} & k_{\max} \geq 2,\\[6pt]
1 & k_{\max} \leq 1,
\end{cases}
\qquad \mathrm{STP} \in [0,1].
\end{equation}

$\mathrm{STP}$ measures concentration in the attractor-rank distribution;
entropy is invariant to which ranks carry the mass. We therefore interpret it
for shallow-dominant attractors with fixed $k_{\max}$. All five configurations
satisfy these conditions: $k_{\max}=4$, the shallowest shell is largest, and
the first two shells contain $97.1$--$99.2\%$ of the absorbed non-initial mass
(\Cref{tab:shells}). Within this family, higher $\mathrm{STP}$ indicates greater
concentration near the violation boundary. $\mathrm{WIN}$ measures the total
winning-region extent.

\subsection{Shield Latitude}
\label{sec:slt}

Let $\mathcal{R} \subseteq V$ be the set of product states a shielded play can
enter, obtained by forward search from $v_0$ through permitted actions only.
Shield latitude is the fraction of the defender's action repertoire the shield
leaves available, averaged over the winning defender states in $\mathcal{R}$:
\begin{equation}
\mathrm{SLT} =
\frac{\displaystyle\sum_{v \,\in\, \calW \cap \mathcal{R} \cap V_D} |\Safe(v)|}
     {|\mathrm{Act}_D| \cdot |\calW \cap \mathcal{R} \cap V_D|}
\;\in\; [0,1].
\end{equation}

$\mathrm{SLT}=1$ when the defender may act freely throughout $\mathcal R$, and
$\mathrm{SLT}\to0$ when only a small fraction of defender actions remains permitted.
The average is taken over $\mathcal R$ to exclude unreachable product states.
Since $\mathcal R\subseteq\calW$, intersecting $\mathcal R$ with $\calW$ does
not change the metric.
$\mathcal{R}$ is also the domain the adaptive metric ranges over, so
$\mathrm{SLT}$ and the adaptive metric are evaluated on the same set of states.
Unlike the two
shell-based metrics, $\mathrm{SLT}$ is computed from the per-state
permitted-action sets rather than from the depth distribution, reading the
solved game at the granularity of individual edges.

\subsection{Defender Dominance}
\label{sec:ddr}

The three static metrics are functionals of $G_\times$ alone. The adaptive
metric, the defender dominance ratio $\mathrm{DDR}$, additionally reads the
dominance function $\mu : S \to [0,1]$ of $\mathcal{M}$ (\Cref{def:casg}).

\begin{definition}[Defender Dominance]
\label{def:ddr}
\begin{equation}
\mathrm{DDR} = \frac{1}{N_{\mathrm{run}} \cdot N_{\mathrm{rep}} \cdot T}
  \sum_{r=1}^{N_{\mathrm{run}}}\ \sum_{e = N_{\mathrm{ep}} - N_{\mathrm{rep}} + 1}^{N_{\mathrm{ep}}}\ \sum_{t=1}^{T} \mu\bigl(s^{(r)}_{e,t}\bigr),
\end{equation}
the mean of $\mu$ under the settled behavior of the two shielded learners of
\Cref{sec:marl}, taken over $N_{\mathrm{run}}$ independent training runs of
$N_{\mathrm{ep}}$ episodes each. Here $N_{\mathrm{rep}} \leq N_{\mathrm{ep}}$ is
the length of the \emph{reporting window}, the trailing episodes of a run over
which the operating point is read; $T$ is the number of steps per episode; and $s^{(r)}_{e,t}$ is the
$\mathcal{M}$-component of the product state visited at step $t$ of episode $e$
of run $r$. Each episode begins at $s_0$ with $\AD$ and $\AAtk$ in their initial
states, and \Cref{sec:respawn} governs re-initialization within an episode.
\end{definition}

$\mathrm{DDR}$ takes values in $[0,1]$. Every contributing trajectory satisfies
$\phiD$ throughout and $\phiA$ within each engagement segment
(\Cref{sec:respawn}), so $\mathrm{DDR}$ measures the quality of a
certified defense under adaptive pressure rather than the rate at which $\phiD$
is violated. $\mathrm{DDR}$ characterizes $\mathcal{C}$ under the declared protocol, with
learning confined to strategies that preserve the certified safety requirement.
The averaging window, the number of runs, and
the training hyperparameters are properties of the experimental protocol (\Cref{sec:experiments}).

\subsection{The Defensibility Fingerprint}
\label{sec:fingerprint}

The four axes plotted on a common radial scale form the \emph{defensibility
fingerprint} of a configuration--specification pair. Each axis is bounded in
$[0,1]$ with endpoints fixed by its definition, and the radius is the metric
value itself. Fingerprints of configurations evaluated under the same
instantiation contract---the same $\mu$, reward form and reset
protocol (\Cref{sec:demo-boundary})---therefore share axes and scale
and may be overlaid directly.

The fingerprint supports coordinate-wise and relational comparison. A
configuration that is larger on every axis is uniformly more defensible under
the shared instantiation contract. Cross-axis comparison identifies joint
movement, opposing movement, and order inversions under controlled changes.
Area or weighted-sum aggregation would discard these relationships.

\section{Instantiation: Network Topological Safety}
\label{sec:instantiation}

This section instantiates the framework of \Cref{sec:sa-framework} on network
topological safety: whether a given network segment, under a given threat
model, admits a defense strategy that keeps a stated safety property
invariant. The resulting product game contains $150{,}000$ states and is evaluated under
controlled architecture and specification changes. Each component of \Cref{def:casg} is fixed in
turn: the game structure becomes a network defense game over host-status
vectors (\Cref{def:ndg}); the defender specification $\phiD$ becomes a
topological invariant over those statuses; and the adversary specification
$\phiA$ becomes a bound on the attacker's destructive budget
(\Cref{sec:dual_spec}). From here the adversary is the \emph{attacker}, states
are host configurations, and actions are network operations.

\begin{figure}[t]
\centering
\begin{tikzpicture}[
    node distance=2.2cm,
    host/.style={draw, thick, rounded corners=3pt, minimum width=1.4cm, minimum height=0.8cm, font=\small\bfseries},
    gwstyle/.style={host, fill=red!15},
    webstyle/.style={host, fill=orange!15},
    wsstyle/.style={host, fill=blue!12},
    dbstyle/.style={host, fill=red!10},
    bkstyle/.style={host, fill=green!15},
    edge/.style={-{Stealth[length=2.5mm]}, thick},
    bypass/.style={-{Stealth[length=2.5mm]}, thick, dashed, color=red!70!black}
]
    \node[gwstyle] (gw) {GW};
    \node[webstyle, right=of gw] (web) {Web};
    \node[wsstyle, above right=0.8cm and 2cm of web] (ws) {WS};
    \node[dbstyle, below right=0.8cm and 2cm of web] (db) {DB};
    \node[bkstyle, right=of db] (bk) {BK};

    \draw[edge] (gw) -- node[above, font=\scriptsize] {FW1} (web);
    \draw[edge] (web) -- node[above left, font=\scriptsize] {FW2} (ws);
    \draw[edge] (web) -- node[below left, font=\scriptsize, pos=0.4] {\textit{p3306}} (db);
    \draw[edge] (ws) -- node[right, font=\scriptsize] {zone} (db);
    \draw[edge] (db) -- node[above, font=\scriptsize] {SSH} (bk);
    \draw[bypass] (gw) to[bend left=22] node[above, font=\scriptsize, color=red!70!black, pos=0.45] {VPN bypass} (ws);
\end{tikzpicture}
\caption{Reference topology. 5 nodes, 6 directed edges. Dashed red edge: VPN bypass path from the gateway to the workstation, representing a forgotten VPN configuration that lets the attacker bypass the web server entirely. GW is internet-facing (attacker entry point); DB and BK are the critical assets protected by $\phiD$.}
\label{fig:topology}
\Description{A node-link network diagram with five labeled host boxes: GW (the internet-facing gateway, on the left), Web (center), WS (workstation, upper right), DB (database, lower right) and BK (backup, far right). Solid directed arrows show allowed connectivity: GW to Web through firewall FW1, Web to WS through firewall FW2, Web to DB over port 3306, WS to DB across a zone boundary, and DB to BK over SSH. A separate dashed arrow curves directly from GW to WS, labeled VPN bypass, depicting a misconfiguration that skips the web server entirely. DB and BK are shaded to mark them as the critical assets the defender must protect.}
\end{figure}

\subsection{Network Model}
\label{sec:network}

\begin{definition}[Network Segment]
A network segment is a directed graph $\Net = (\calH, \Links)$ whose vertices $\calH$
are hosts and whose directed edges $\Links \subseteq \calH \times \calH$ are the
\emph{links} along which compromise can propagate: firewall rules, service
dependencies, or misconfigurations.
\end{definition}

The reference segment (\Cref{fig:topology}) has five hosts and six directed
edges. The GW\,$\to$\,WS edge is its structurally significant feature: a
single hop from the attacker's entry point to the interior, bypassing the web
server that the remaining edges route through.

Each host $h \in \calH$ carries a status from $\St = \{C, X, D, I, Z\}$:
Clean~(C)---no attacker presence; Compromised~(X)---attacker has access,
defender unaware; Detected~(D)---attacker has access, defender aware;
Isolated~(I)---disconnected by defender action; or Destroyed~(Z)---taken
offline by attacker action, recoverable by defender repair. The five statuses
follow the compromise--detection--response structure standard in intrusion
analysis~\cite{hutchins2011}: initial compromise, discovery by the defender,
and the two responses available to each side, with recovery possible from
both.

\begin{definition}[Network Defense Game]
\label{def:ndg}
The network defense game $\Mnet$ is the game structure $\mathcal{M}$ of
\Cref{def:casg} instantiated on $\Net$, with
\begin{itemize}
    \item $S = \St^{|\calH|} \times \{\turnD, \turnA\}$: host status vector
    paired with a turn indicator, the turn component supplying the partition
    $S_D = \St^{|\calH|} \times \{\turnD\}$ and
    $S_A = \St^{|\calH|} \times \{\turnA\}$;
    \item $s_0$: one host compromised, all others clean, attacker to move;
    \item $\mathrm{Act}_D = \{\textsc{Noop}, \textsc{Monitor}, \textsc{Isolate}, \textsc{Restore}, \textsc{Fix}\} \times \calH$;
    \item $\mathrm{Act}_A = \{\textsc{Noop}, \textsc{Spread}, \textsc{Destroy}\} \times \calH$;
    \item $\delta: S \times (\mathrm{Act}_D \cup \mathrm{Act}_A) \to S$: deterministic, alternating the turn;
    \item $\mu(s) = n_C(s)/|\calH|$, the fraction of clean hosts, where
    $n_C(s) = |\{h \in \calH : \mathrm{status}(h) = C\}|$.
\end{itemize}
\end{definition}

Under this choice of $\mu$ the zero-sum reward of \Cref{sec:marl} reduces to
$R_D(s) = 2\mu(s) - 1 = \bigl(n_C(s) - n_{\neg C}(s)\bigr)/|\calH|$, where
$n_{\neg C} = |\calH| - n_C$: the defender is rewarded for the margin of clean
hosts over unclean ones.

Each action targets one host. \textsc{Monitor} reveals whether a host is
compromised ($X \to D$). \textsc{Isolate} disconnects a host. \textsc{Restore}
reconnects an isolated host to clean status. \textsc{Fix} repairs a destroyed
host to clean. \textsc{Spread} propagates compromise to a clean target $h$,
succeeding iff there exists $u$ with $\mathrm{status}(u) \in \{X, D\}$ and
$(u,h) \in \Links$. \textsc{Destroy} takes a compromised or detected host offline,
recoverable via \textsc{Fix}. Both players have a \textsc{Noop}.
\Cref{tab:instance} gives the full transition semantics.

All $|\mathrm{Act}_D| = 25$ defender actions and $|\mathrm{Act}_A| = 15$
attacker actions are selectable at every state. Whether an action changes the state is a property of $\delta$ and the
current configuration, not of the action set; the uniform action space
simplifies the product construction and leaves the game graph total.

\begin{table}[t]
\centering
\caption{Left: hosts of the reference segment; DB and BK are the critical
assets named by $\phiD$ (\Cref{def:phid}). Right: transition semantics, each
cell giving the resulting host status when the action targets a host in that
status. $^*$\textsc{Spread} to Clean requires an adjacent compromised or
detected neighbor on a directed edge $(u,h) \in \Links$.}
\label{tab:instance}
\begin{minipage}[t]{0.34\textwidth}
\centering
\begin{tabular}{@{}ll@{}}
\toprule
Host & Role \\
\midrule
GW  & Gateway; entry point \\
Web & Web server \\
WS  & Workstation \\
DB  & Database \\
BK  & Backup \\
\bottomrule
\end{tabular}
\end{minipage}
\hfill
\begin{minipage}[t]{0.62\textwidth}
\centering
\begin{tabular}{@{}lccccc@{}}
\toprule
& C & X & D & I & Z \\
\midrule
\multicolumn{6}{l}{\textit{Defender Actions}} \\
\textsc{Monitor} & C & D & D & I & Z \\
\textsc{Isolate} & I & I & I & I & Z \\
\textsc{Restore} & C & X & D & C & Z \\
\textsc{Fix} & C & X & D & I & C \\
\midrule
\multicolumn{6}{l}{\textit{Attacker Actions}} \\
\textsc{Spread} & X$^*$ & X & D & I & Z \\
\textsc{Destroy} & C & Z & Z & I & Z \\
\bottomrule
\end{tabular}
\end{minipage}
\end{table}

\subsection{Dual Temporal Logic Specifications}
\label{sec:dual_spec}

We now fix the two specifications for this instance. $\phiD$ states a
topological invariant the defender must preserve; $\phiA$ bounds the
attacker's destructive budget, instantiating the asymmetric enforcement of
\Cref{sec:asym-general}.

\subsubsection{Defender Safety Specification}

\begin{definition}[Defender Specification $\phiD$]
\label{def:phid}
\begin{equation}
\phiD: \square\,\neg\bigl(\mathrm{DB} \in \{X,D,Z\} \;\wedge\; \mathrm{BK} \in \{X,D,Z\}\bigr) \;\wedge\; \square\,(\mathrm{active} \geq 3)
\end{equation}
where $\mathrm{active}(s) = |\{h \in \calH : \mathrm{status}(h) \notin \{I,Z\}\}|$.
\end{definition}

In plain terms: (1)~the database and backup server are never simultaneously in a compromised, detected, or destroyed state, and (2)~at least three hosts remain operational at all times. This is a conjunction of two safety properties, each compiled into a separate DFA:

\begin{itemize}
    \item $\AD^{(1)}$: tracks how many of $\{\mathrm{DB}, \mathrm{BK}\}$ are in a ``bad'' state ($\{X, D, Z\}$). Three states: $q_0$ (neither bad), $q_1$ (exactly one bad), $\mathrm{viol}$ (both bad, absorbing). Accepts in $\{q_0, q_1\}$.
    \item $\AD^{(2)}$: checks whether the active host count falls below 3. Two states: $\mathrm{safe}$ and $\mathrm{viol}$. Accepts in $\{\mathrm{safe}\}$.
\end{itemize}

The product $\AD = \AD^{(1)} \otimes \AD^{(2)}$ has $|Q_D| = 6$ states with $|F_D| = 2$ accepting states: $(q_0, \mathrm{safe})$ and $(q_1, \mathrm{safe})$. The remaining four states represent violation of at least one safety clause.

\begin{figure}[ht]
\centering
\begin{minipage}{0.60\textwidth}
\centering
\begin{tikzpicture}[dfa]
  \node[state, initial, initial text={}, accepting] (q0) {$q_0$};
  \node[state, accepting, right=of q0]              (q1) {$q_1$};
  \node[state, viol, right=of q1]                   (qv) {$\mathit{viol}$};

  \path[->]
    (q0) edge[loop above]    node {$0$}        (q0)
    (q0) edge[bend left=18]  node[above] {$1$} (q1)
    (q1) edge[bend left=18]  node[below] {$0$} (q0)
    (q1) edge[loop above]    node {$1$}        (q1)
    (q1) edge                node[above] {$2$} (qv)
    (q0) edge[bend right=32] node[below] {$2$} (qv)
    (qv) edge[loop above]    node {$\top$}    (qv);
\end{tikzpicture}\\[2pt]
\textbf{(a)}~$\mathcal{A}_D^{(1)}$
\end{minipage}
\hfill
\begin{minipage}{0.36\textwidth}
\centering
\begin{tikzpicture}[dfa, node distance=3.1cm]
  \node[state, initial, initial text={}, accepting] (sf) {$\mathit{safe}$};
  \node[state, viol, right=of sf]                    (vv) {$\mathit{viol}$};

  \path[->]
    (sf) edge[loop above] node {$\mathrm{active}\!\ge\!3$} (sf)
    (sf) edge             node[above] {$\mathrm{active}\!<\!3$} (vv)
    (vv) edge[loop above] node {$\top$} (vv);
\end{tikzpicture}\\[2pt]
\textbf{(b)}~$\mathcal{A}_D^{(2)}$
\end{minipage}
\caption{The two defender safety automata composing
$\mathcal{A}_D = \mathcal{A}_D^{(1)} \otimes \mathcal{A}_D^{(2)}$ for $\varphi_D$
(Definition~\ref{def:phid}). \textbf{(a)} reads the number of critical assets
$\{\mathrm{DB},\mathrm{BK}\}$ currently bad ($\in\{X,D,Z\}$): $q_0$ none, $q_1$
exactly one, $\mathit{viol}$ both. \textbf{(b)} reads whether at least three hosts
remain active. Double circles are accepting; shaded circles are absorbing
violation sinks ($\top$ $=$ any input).}
\label{fig:dfa-defender}
\Description{Two small state-transition diagrams side by side, each a deterministic safety automaton. The left automaton has three states in a row: an initial accepting state for zero critical assets bad, a middle accepting state for exactly one, and a shaded absorbing non-accepting sink for both, with forward arrows advancing toward the sink as more assets go bad and a self-loop on the sink. The right automaton has two states: an initial accepting state labeled safe with a self-loop for at least three hosts active, and a shaded absorbing sink labeled viol reached by a single forward arrow when fewer than three remain active, with a self-loop on the sink. Accepting states are drawn as double circles; the sinks are reached irreversibly.}
\end{figure}

\subsubsection{Attacker Constraint Specification}

\begin{definition}[Attacker Specification $\phiA$]
\label{def:phia}
$\phiA$ admits exactly those attacker behaviors that select at most two
\textsc{Destroy} actions:
\begin{equation}
\phiA: \square\,\bigl(\mathrm{destroys} \leq 2\bigr),
\qquad
\mathrm{destroys} = \bigl|\{\,i : a_i \in \textsc{Destroy} \times \calH\,\}\bigr|,
\end{equation}
where $a_i$ is the $i$-th attacker action of the play.
\end{definition}

$\phiA$ compiles into the four-state counter $\AAtk$ of
\Cref{fig:dfa-attacker}, accepting in $\{q_0, q_1, q_2\}$. Each selected
\textsc{Destroy} action consumes one budget unit, including actions that leave
the host state unchanged.

Within the product game $\phiA$ binds over the whole play: $\AAtk$ has no
reset transition, and the certificate accordingly quantifies over attacker
behaviors that never select a third \textsc{Destroy}. Engagement structure
enters only at the adaptive layer. This instance takes the re-initialization
set $S_{\mathrm{init}}$ of \Cref{sec:respawn} to be the all-clean
configurations, with $p = 0.1$: when every host is clean the attacker respawns
at the gateway and both automata reset to their initial states, beginning a new
engagement. Training episodes are $1{,}000$ steps long and each already begins
afresh from $s_0$; without the within-episode reset, two destroys would be spent
in the opening moves and the rest of every episode would measure an attacker that
can no longer destroy, inflating the dominance the adaptive layer reports.

\begin{figure}[ht]
\centering
\begin{tikzpicture}[dfa, node distance=2.9cm]
  \node[state, initial, initial text={}, accepting] (a0) {$q_0$};
  \node[state, accepting, right=of a0]              (a1) {$q_1$};
  \node[state, accepting, right=of a1]              (a2) {$q_2$};
  \node[state, viol, right=of a2]                   (av) {$\mathit{viol}$};

  \path[->]
    (a0) edge[loop above] node {$\neg$\textsc{dst}}  (a0)
    (a0) edge             node[above] {\textsc{dst}} (a1)
    (a1) edge[loop above] node {$\neg$\textsc{dst}}  (a1)
    (a1) edge             node[above] {\textsc{dst}} (a2)
    (a2) edge[loop above] node {$\neg$\textsc{dst}}  (a2)
    (a2) edge             node[above] {\textsc{dst}} (av)
    (av) edge[loop above] node {$\top$}       (av);
\end{tikzpicture}
\caption{Attacker constraint automaton $\mathcal{A}_A$ for $\varphi_A$
(Definition~\ref{def:phia}): at most two Destroy actions over the play.
\textsc{dst} $=$ a \textsc{Destroy} action, $\neg$\textsc{dst} $=$ any other
attacker action, $\top$ $=$ any input; $q_i$ records $i$ Destroys spent. $\mathcal{A}_A$ advances on every Destroy the attacker
\emph{selects}---including no-op Destroys---so attempted destruction is priced like
successful destruction.}
\label{fig:dfa-attacker}
\Description{A single deterministic automaton drawn as four states in a horizontal row: an initial accepting state q0, two further accepting states q1 and q2, and a shaded absorbing non-accepting sink labeled viol. A forward arrow labeled DST advances from each state to the next, counting Destroy actions spent; every non-sink state also has a self-loop labeled not-DST for any other attacker action; the sink has a self-loop for any input. Reaching the sink after the third Destroy is irreversible. Accepting states are drawn as double circles.}
\end{figure}

\paragraph{Asymmetric enforcement in the instance.}
The two specifications instantiate the mechanism of \Cref{sec:asym-general}:
$\phiD$ seeds the unsafe set of the product game, while $\phiA$ filters the
attacker's successors during attractor computation. The asymmetry is concrete
here. Only $q_2$ has an immediate \textsc{Destroy} successor outside $F_A$, but
nothing prevents the attacker from selecting \textsc{Destroy} repeatedly, so
every attacker position can leave $F_A$ in finitely many moves. Seeding
attacker-violation states into $U$ would therefore make every attacker position
a predecessor of $U$, and the attractor would absorb the state space.

\subsection{Product Game Construction}

Paired with these specifications, the triple
$\mathcal{C}_{\mathrm{net}} = (\Mnet, \phiD, \phiA)$ is an asymmetrically
constrained safety game. Its product game tracks the network configuration,
the current turn, the defender's safety status, and the attacker's budget
status in a single state:
\begin{equation}
v = \bigl(\underbrace{\mathrm{status}(h_1), \ldots, \mathrm{status}(h_{|\calH|})}_{\text{host statuses}},\ \underbrace{\tau}_{\text{turn}},\ \underbrace{q_D}_{\AD},\ \underbrace{q_A}_{\AAtk}\bigr)
\end{equation}
The set of all such states is the product-game position set $V$ defined in
\Cref{def:product}. For
the reference segment ($|\calH| = 5$ hosts, $|\St| = 5$ statuses, $|Q_D| = 6$,
$|Q_A| = 4$):
\begin{equation}
|V| = |\St|^{|\calH|} \times 2 \times |Q_D| \times |Q_A|
    = 5^5 \times 2 \times 6 \times 4 = 150{,}000 .
\end{equation}
Of these, exactly $100{,}000$ carry $q_D \notin F_D$ and so violate $\phiD$
independently of any game dynamics; these form the initial unsafe set
$\Attr_0$. The remaining $50{,}000$ are partitioned by the attractor
computation into the winning region $\calW$ and the states absorbed during
fixed-point iteration.

Transitions are derived mechanically. In this instance $\AD$ reads the
host-status vector of the successor state and $\AAtk$ reads the action label,
advancing when and only when the attacker selects a \textsc{Destroy}. No
defender action is a \textsc{Destroy}, so $\AAtk$ is constant on defender
moves.

\subsubsection{A short play}
Five moves from the initial configuration. Each row shows the state
\emph{after} the indicated action: the host-status vector
$\langle \mathrm{GW},\mathrm{Web},\mathrm{WS},\mathrm{DB},\mathrm{BK}\rangle$,
the turn $\tau \in \{\turnD, \turnA\}$ (next to move), and the automaton states
$q_D = (q^{(1)},q^{(2)})$ and $q_A$.

\begin{center}
\small
\begin{tabular}{c l c c c c c}
\toprule
\# & Action & $\langle\text{GW,Web,WS,DB,BK}\rangle$ & $\tau$
   & $\mathcal{A}_D^{(1)}$ & $\mathcal{A}_D^{(2)}$ & $\mathcal{A}_A$ \\
\midrule
$0$ & \emph{initial state}            & $X\,C\,C\,C\,C$ & \turnA & $q_0$ & $\mathit{safe}$ & $q_0$ \\
$1$ & Attacker: \texttt{Spread(WS)}   & $X\,C\,X\,C\,C$ & \turnD & $q_0$ & $\mathit{safe}$ & $q_0$ \\
$2$ & Defender: \texttt{Monitor(WS)}  & $X\,C\,D\,C\,C$ & \turnA & $q_0$ & $\mathit{safe}$ & $q_0$ \\
$3$ & Attacker: \texttt{Spread(DB)}   & $X\,C\,D\,X\,C$ & \turnD & $\mathbf{q_1}$ & $\mathit{safe}$ & $q_0$ \\
$4$ & Defender: \texttt{Isolate(DB)}  & $X\,C\,D\,I\,C$ & \turnA & $\mathbf{q_0}$ & $\mathit{safe}$ & $q_0$ \\
$5$ & Attacker: \texttt{Destroy(WS)}  & $X\,C\,Z\,I\,C$ & \turnD & $q_0$ & $\mathit{safe}$ & $\mathbf{q_1}$ \\
\bottomrule
\end{tabular}
\end{center}

\noindent
Move~1 uses the GW\,$\to$\,WS bypass, reaching the interior without touching
Web. Move~3 puts one critical asset in a bad state, advancing
$\mathcal{A}_D^{(1)}$ to $q_1$; the state remains accepting, since $\phiD$
requires \emph{both} critical assets to be bad before it is violated. Move~4
returns DB to the good set and $\mathcal{A}_D^{(1)}$ to $q_0$, exercising the
recovery edge of \Cref{fig:dfa-defender}a. Move~5 spends one unit of attacker
budget: $\mathcal{A}_A$ advances to $q_1$ while $\mathcal{A}_D^{(2)}$ stays
$\mathit{safe}$, three hosts remaining active.

\subsection{Certificate for the Reference Configuration}
\label{sec:instance-certificate}

For the reference configuration, the attractor converges in four iterations.
\begin{center}
\begin{tabular}{@{}lrr@{}}
\toprule
Iteration & $|\Attr_k|$ & Shell $|\Sh_k|$ \\
\midrule
$\Attr_0$ (initial violations) & 100{,}000 & --- \\
$\Attr_1$ & 119{,}238 & 19{,}238 \\
$\Attr_2$ & 125{,}630 & 6{,}392 \\
$\Attr_3$ & 125{,}886 & 256 \\
$\Attr_4$ (fixed point) & 126{,}270 & 384 \\
\midrule
Winning region $|\calW|$ & 23{,}730 & (15.82\% of $|V|$) \\
\bottomrule
\end{tabular}
\end{center}

The initial product state $v_0 \in \calW$: the reference segment is
defensible under $\phiD$ against every $\phiA$-admissible attacker. The shell
sizes are read in \Cref{sec:experiments}, where the perturbed configurations
supply a basis for comparison.

\subsection{The Baseline Fingerprint}
\label{sec:baseline-fingerprint}
The four metrics of \Cref{sec:metrics}, evaluated on the solved reference game,
compose into the defensibility fingerprint of \Cref{fig:fingerprint}: the three
static axes read from the fixed point and shield, the adaptive axis from
shielded play within $\calW$ (values in \Cref{tab:results}). The absolute
construction is fixed in \Cref{sec:fingerprint}. This is the per-configuration
object that \Cref{sec:experiments} places side by side across the five
configurations.

\begin{figure}[ht]
\centering
\includegraphics[width=0.65\columnwidth]{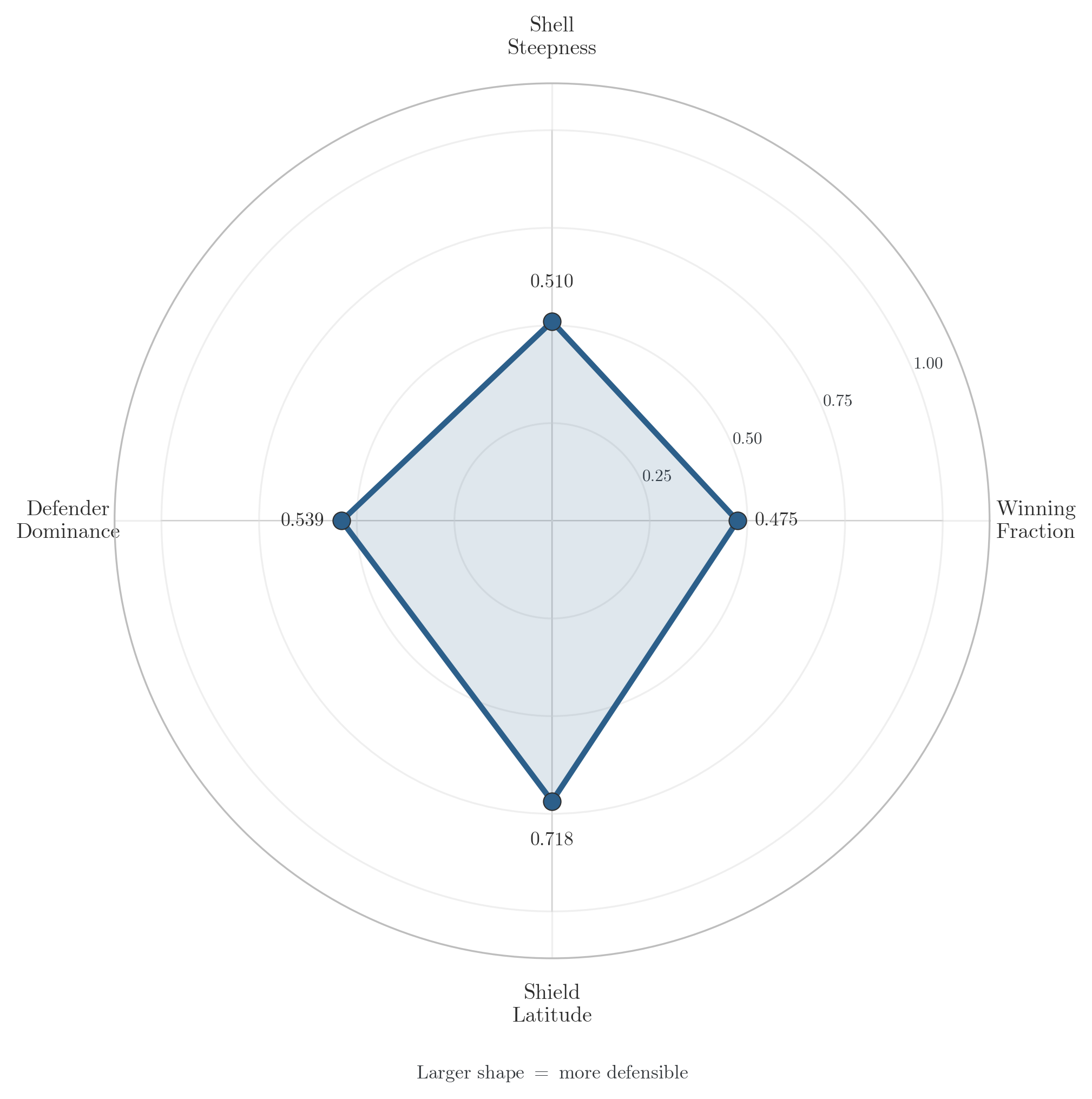}
\caption{Defensibility fingerprint of the baseline configuration (Case~1) on the
four axes of \Cref{sec:metrics}. Each axis spans $[0,1]$ on an absolute scale
with the rim at $1$; no cross-case rescaling is applied. Every axis is oriented
so that larger values indicate greater defensibility; distance from the center
gives the metric value and the polygon records the coordinate-wise profile
(\Cref{sec:fingerprint}).}
\label{fig:fingerprint}
\Description{A radar chart with four axes arranged around a center, one per defensibility metric: winning fraction, shell steepness, shield latitude and defender dominance. A single closed polygon connects the baseline configuration's value on each axis. Every axis runs from zero at the center to one at the rim on the same absolute scale; distance from the center gives each metric value, and the polygon records the coordinate-wise profile.}
\end{figure}

\section{Controlled Analysis of Network Defensibility}
\label{sec:experiments}

We evaluate the framework on the reference network and four controlled
perturbations. The explicit-state model permits complete evaluation of all five
configurations.

\subsection{Instance and Protocol}
\label{sec:whatif}

The reference configuration is perturbed in four ways (\Cref{tab:cases}). Two
perturbations alter the arena and leave both specifications fixed; two alter a
specification and leave the arena fixed. Separating the classes allows each
response to be attributed to one or the other. Both arena perturbations preserve
the adversary's ability to reach every host.

\begin{table}[ht]
\centering
\caption{The five configurations.}
\label{tab:cases}
\begin{tabular}{@{}clp{6.4cm}@{}}
\toprule
Case & Perturbs & Description \\
\midrule
1 & --- & Reference: $\phiD$ = no dual-critical-asset failure $\wedge$ active $\geq 3$; $\phiA$ = at most two destroys \\
2 & Arena & Fully connected: bidirectional edges between all non-BK hosts, BK reachable only from DB. Directed edges rise from 6 to 13. \\
3 & Spec ($\phiA$) & Unlimited \textsc{Destroy}: no attacker action is charged against the budget, so $\AAtk$ never leaves $q_0$ and $\Legal(v) = E(v)$ at every attacker vertex a play can reach. The counter automaton is retained in the product, so $|V|$ remains 150{,}000 and the case stays comparable state-for-state with the reference. \\
4 & Spec ($\phiD$) & Availability clause relaxed to active $\geq 2$ \\
5 & Arena & VPN bypass edge removed \\
\bottomrule
\end{tabular}
\end{table}

The static readings are functionals of the solved game and are deterministic.
The adaptive reading is trained under \Cref{tab:hyperparams} on ten seeds and
reported over the final $200$ episodes (the L200 window) with a $95\%$ Student's
$t$ interval, $\bar{x} \pm t_{9,0.025}\,s/\sqrt{n}$. Configurations are
independent and run in parallel; one configuration completes in about ten
minutes on commodity hardware. \Cref{tab:results} collects the four readings for
all five configurations, and the remaining subsections draw on it throughout.

\begin{table}[b]
\centering
\caption{Training parameters.}
\label{tab:hyperparams}
\begin{tabular}{@{}llll@{}}
\toprule
$\alpha$ & 0.05 & Episodes & 3{,}000 \\
$\gamma$ & 0.95 & Steps per episode & 1{,}000 \\
$\varepsilon$ & $\max(0.1,\,0.5 - e/6000)$ & Engagement reset $p$ & 0.1 \\
$Q$ initialization & $5.0$, optimistic & Window $N_{\mathrm{rep}}$ & final 200 episodes \\
Runs $N_{\mathrm{run}}$ & 10 seeds & Episode start & $s_0$, both automata reset \\
\bottomrule
\end{tabular}
\end{table}

\begin{table}[t]
\centering
\caption{The four readings across the five configurations. All axes lie in
$[0,1]$, larger meaning more defensible. The static axes are exact; $\mathrm{DDR}$
is the L200 ten-seed mean
with its $95\%$ interval. $|\calW|$ is the winning-region size and W\% its share
of the $150{,}000$ product states.}
\label{tab:results}
\small
\setlength{\tabcolsep}{5pt}
\begin{tabular}{@{}cl cccc rr c@{}}
\toprule
Case & Description & $\mathrm{WIN}$ & $\mathrm{STP}$ & $\mathrm{SLT}$ & $\mathrm{DDR}$ & $|\calW|$ & W\% & 95\% CI \\
\midrule
1 & Reference             & 0.4746 & 0.5102 & 0.7180 & 0.539 & 23{,}730 & 15.82\% & [0.524,\ 0.553] \\
2 & Fully connected       & 0.4662 & 0.5125 & 0.7067 & 0.227 & 23{,}310 & 15.54\% & [0.216,\ 0.237] \\
3 & Unlimited destroys    & 0.4105 & 0.4676 & 0.4625 & 0.475 & 20{,}526 & 13.68\% & [0.472,\ 0.479] \\
4 & Active $\geq 2$       & 0.6793 & 0.6165 & 0.8180 & 0.779 & 33{,}966 & 22.64\% & [0.774,\ 0.783] \\
5 & No bypass edge        & 0.4772 & 0.5235 & 0.7327 & 0.807 & 23{,}858 & 15.91\% & [0.794,\ 0.821] \\
\bottomrule
\end{tabular}
\end{table}

\subsection{Certification and Operational Quality Separate}
\label{sec:demo-certificate}

All five configurations are certified defensible: the initial product state lies
in $\calW$ in each, so in each the defender can maintain $\phiD$ against every
behavior $\phiA$ admits. Yet two of the certified configurations, differing
only in their arenas and by $548$ of $150{,}000$ product states in their winning
regions, hold $22.7\%$ and $80.7\%$ of the network respectively once both
parties adapt. Behind identical verdicts and nearly identical winning regions
sit defenses of entirely different operational quality.

The pattern is general: existence of a safe strategy, and the operational quality
attained by shield-constrained adaptive play under a declared protocol, are
separate quantities, and systems indistinguishable by the first can lie at
opposite ends of the second. Certification establishes that a safe way of
inhabiting the game exists; the adaptive layer asks what operating condition the
certified region supports when it is probed. Certification supplies the
existence result exhaustively: because finite safety games are determined, a
negative verdict establishes the absence of a defender strategy rather than a
failure to find one.

The separation becomes observable because the framework evaluates both
quantities on the same solved game: the certificate and $\mathrm{WIN}$ from the
fixed point, $\mathrm{DDR}$ from adaptive play confined to the region the fixed
point certifies. Every trajectory behind the $22.7\%$ figure is safe by
construction; what has collapsed is not safety but the price of keeping it.

Each reading alone answers its own question. Producing them from one object is
what turns two answers into a comparison, and the rest of this section is a
sequence of such comparisons.

\subsection{Depth Localization of an Edge Perturbation}
\label{sec:demo-structure}

Deleting the bypass edge leaves the rank-1 and rank-2 shells unchanged at
$19{,}238$ and $6{,}392$ positions, respectively (\Cref{tab:shells}). The
rank-3 shell decreases from $256$ to $224$, and rank 4 from $384$ to $288$.
The edge therefore affects only forcing sequences of depth three or greater.

\begin{table}[t]
\centering
\caption{Attractor shells. $\Sh_k$ is the set of positions at rank $k$, i.e.\
positions from which the adversary can force violation in exactly $k$ moves.
$\Attr^{*}$ is the converged attractor.}
\label{tab:shells}
\small
\begin{tabular}{@{}cl rrrr rr@{}}
\toprule
Case & Description & $|\Sh_1|$ & $|\Sh_2|$ & $|\Sh_3|$ & $|\Sh_4|$ & $|\Attr^{*}|$ & $|\calW|$ \\
\midrule
1 & Reference          & 19{,}238 & 6{,}392 & 256 & 384 & 126{,}270 & 23{,}730 \\
2 & Fully connected    & 19{,}418 & 6{,}696 & 240 & 336 & 126{,}690 & 23{,}310 \\
3 & Unlimited destroys & 20{,}258 & 8{,}352 & 336 & 528 & 129{,}474 & 20{,}526 \\
4 & Active $\geq 2$    & 12{,}914 & 2{,}992 &  64 &  64 & 116{,}034 & 33{,}966 \\
5 & No bypass edge     & 19{,}238 & 6{,}392 & 224 & 288 & 126{,}142 & 23{,}858 \\
\bottomrule
\end{tabular}
\end{table}

Specification changes leave a different signature in the same decomposition.
Relaxing the availability clause (Case~4) removes $33\%$ of the rank-1 shell but
$80\%$ of the combined rank-3 and rank-4 mass, from $640$ positions to $128$:
the relaxation dissolves long forcing chains far faster than immediate exposure.
Two changes that a scalar margin records as ``the region moved'' act at opposite
ends of the depth axis.

In general, a change to a model has a depth profile as well as a magnitude: it
may create or destroy immediate hazards, or act only on long strategic
horizons, and interventions with similar aggregate effect can differ entirely in
where that effect falls. The profile is exposed here because the attractor is computed synchronously,
which makes each state's rank equal to the number of moves within which violation
can be forced---forced strategic depth under alternating play, not hop distance
in the network graph. Reading the fixed point shell by shell rather than as a set
then localizes any perturbation by depth without requiring an additional game
solution (\Cref{sec:shells}). \Cref{fig:decomposition} shows the resulting partition for
all five configurations.

\begin{figure}[ht]
\centering
\includegraphics[width=\textwidth]{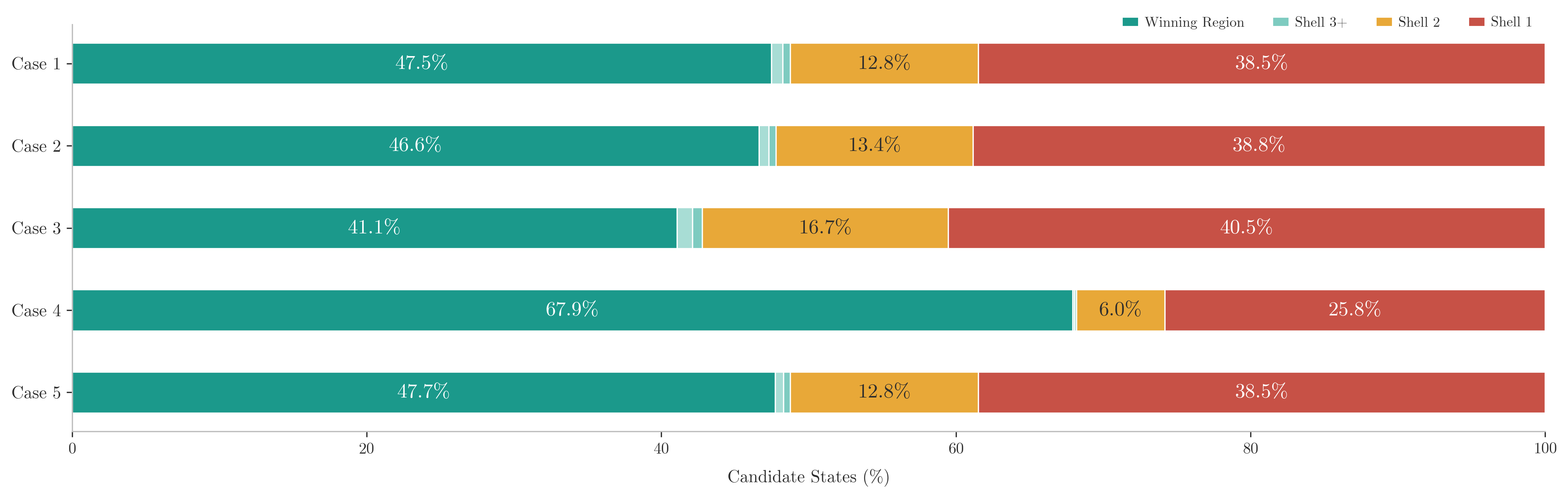}
\caption{Winning region and attractor shells as a partition of the
$|V| - |U| = 50{,}000$ states not already violating $\phiD$.}
\label{fig:decomposition}
\Description{A horizontal stacked bar chart with one bar per configuration, Cases 1 through 5, each bar spanning the fifty thousand product states that do not already violate the defender specification, normalized to 100 percent. Each bar is split into four coloured segments: the defender winning region, and the attractor shells at rank one, rank two, and rank three-or-more, with segment width proportional to the share of states in that region. Case 4, the relaxed availability clause, shows the widest winning-region segment and Case 3, unlimited destroys, the narrowest, while the rank-one shell dominates the absorbed mass in every case. The absolute winning-region size and its percentage of the full state space are annotated beside each bar.}
\end{figure}

\subsection{Operational Divergence under Similar Structural Margins}
\label{sec:demo-operation}

The two arena perturbations move every static reading by under $2.7\%$---mean
absolute relative movement $1.16\%$ for $\mathrm{WIN}$, $1.53\%$ for
$\mathrm{STP}$ and $1.81\%$ for $\mathrm{SLT}$---and move the operating point by
$58.0$ percentage points: $22.7\%$ of the network
held in the fully connected arena against $80.7\%$ with the bypass removed
(\Cref{fig:dominance}). Both cases remain close to the reference on all static
metrics but attain the lowest and highest $\mathrm{DDR}$ values among the five
configurations.

\begin{figure}[t]
\centering
\includegraphics[width=\textwidth]{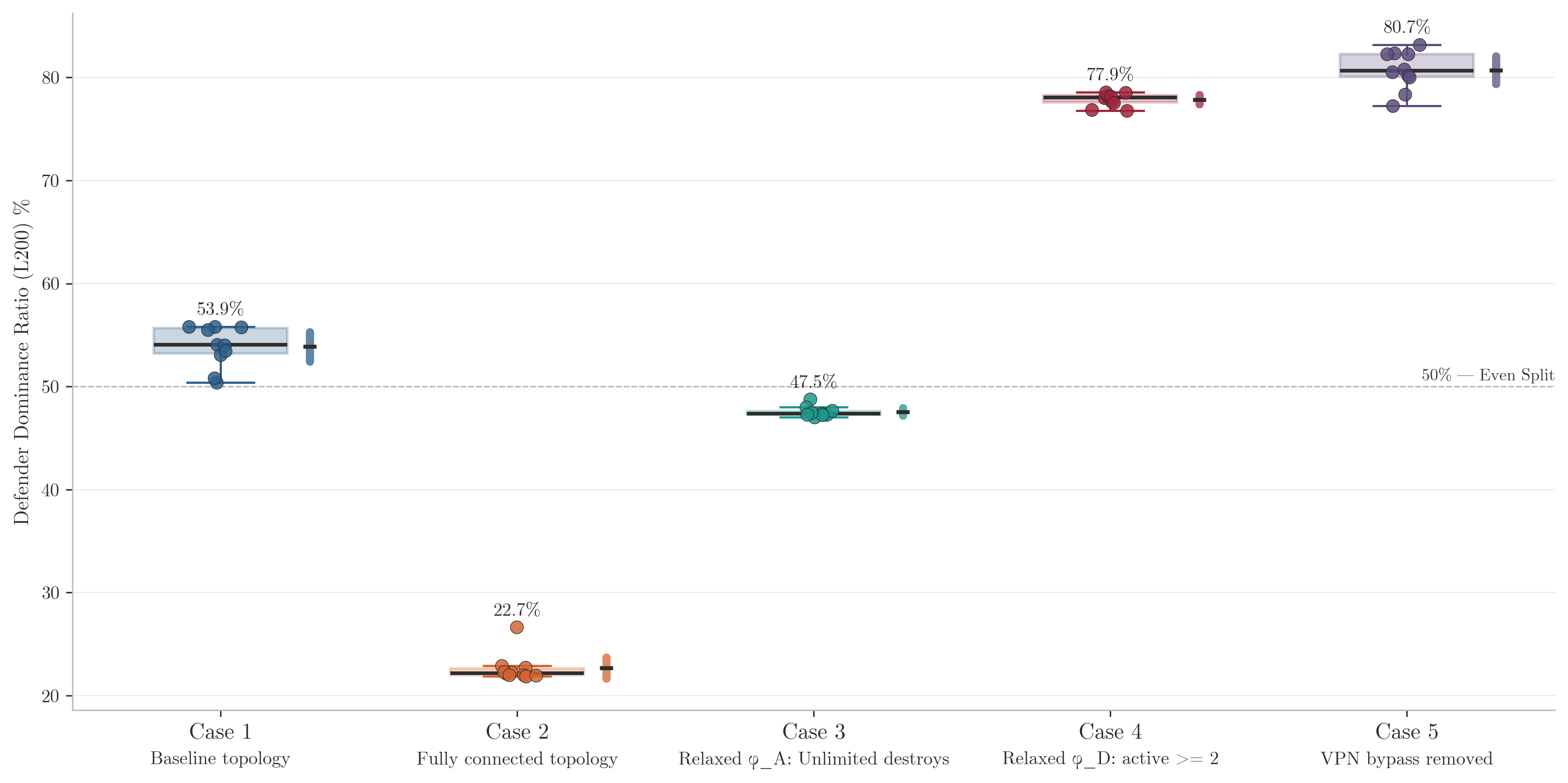}
\caption{Defender dominance over the L200 window. Each point is one seed
($n=10$); boxes show the interquartile range, the narrow box the $95\%$
interval. The dashed line marks an even split.}
\label{fig:dominance}
\Description{A box-and-whisker plot with one box per configuration, Cases 1 through 5, summarizing the defender dominance ratio over the final two hundred episodes across ten seeds, with the individual per-seed values overlaid as scattered points and a separate narrow box marking the 95 percent confidence interval. The vertical axis is the defender dominance ratio as a percentage, so a higher box means the defender holds more of the network; a dashed horizontal reference line at 50 percent marks an even split. The cases separate sharply: Case 2, fully connected, sits far below the even-split line near 23 percent, while Cases 4 and 5 sit well above it near 78 and 81 percent, and the tight clustering of points within each case reflects narrow inter-seed variation.}
\end{figure}

The static summaries move little while adaptive dominance changes sharply.
\Cref{sec:demo-structure} also locates the bypass deletion in the deep attractor
shells. Defender actions are independent of the edge relation, whereas
\textsc{Spread} is edge-gated (\Cref{def:ndg}). Reachability-preserving edge
changes therefore primarily affect compromise propagation time rather than the
existence of defensive responses. In absolute units the mean movement per static
axis across the two arena cases is $0.006$--$0.013$, against $0.290$ for
dominance.

The gap is measurable because the two layers read different objects of the same
game: the fixed point reads its winning condition, adaptive play reads the
dominance function over trajectories inside $\calW$. \Cref{fig:trajectories}
shows those trajectories---the per-episode mean of $\mu$ itself, the quantity
$\mathrm{DDR}$ averages. Every curve is flat inside the reporting window, and
the five differ in shape as well as level; both facts are used in
\Cref{sec:demo-trajectory}. Together, the certificate and $\mathrm{DDR}$ show
that both configurations are defensible but differ fourfold in holding burden.
Writing the holding burden as
$B_{\mathrm{hold}} = 1 - \mathrm{DDR}$, the share of the network the defender
does not hold, the two sit at $0.773$ and $0.193$---a ratio of $4.00$.

\begin{figure}[ht]
\centering
\includegraphics[width=\textwidth]{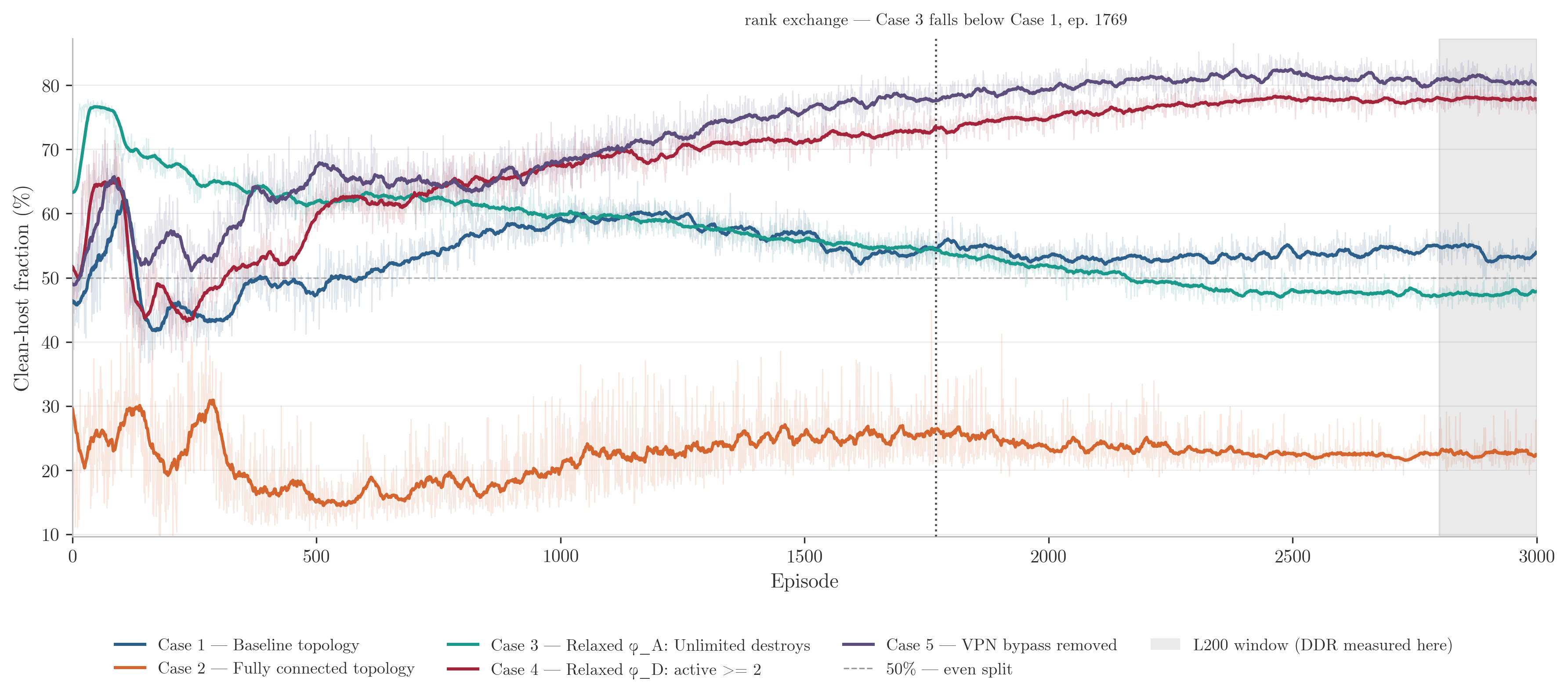}
\caption{Seed-mean clean-host fraction $\mu$ against training episode, ten seeds
per configuration, under a trailing moving average of 25 episodes with the
unsmoothed seed mean drawn faintly behind each curve. The shaded band marks the
final 200 episodes over which $\mathrm{DDR}$ is measured. The dotted line marks
the episode at which Case~3 falls below the reference and remains below it.}
\label{fig:trajectories}
\Description{A line chart with training episode on the horizontal axis running from zero to three thousand and the mean fraction of clean hosts on the vertical axis, expressed as a percentage. One curve per configuration, Cases 1 through 5, each the mean over ten seeds and lightly smoothed, with the unsmoothed mean shown faintly behind it. A shaded band at the right marks the final two hundred episodes over which defender dominance is measured, and every curve is flat within that band. A horizontal dashed line marks an even fifty percent split. Case 3 begins highest of all configurations, rising above three quarters within the first hundred episodes, then declines steadily for the remainder of training and crosses below the reference, which is marked with a dotted vertical line. Cases 4 and 5 rise after an early dip and settle highest; Case 2 remains far below every other configuration throughout.}
\end{figure}

\subsection{Inversions between Structural and Operational Orderings}
\label{sec:demo-inversion}

The harshest threat model in the study does not produce the weakest operational
position. Case~2 is formally the stronger configuration of the pair
(2,\,3) on every static axis---$\mathrm{WIN}$ by $0.056$, $\mathrm{STP}$ by
$0.045$, $\mathrm{SLT}$ by $0.244$---and Case~3, whose adversary operates under
no budget at all, holds more than twice the network: $47.5\%$ against $22.7\%$,
a $24.8$-point gap with disjoint intervals. The formal ordering and the
operational ordering of this pair point in opposite directions.

The same inversion decides a design choice in the pair (4,\,5). Relaxing the
availability requirement dominates removing the bypass edge on every static
axis, $\mathrm{WIN}$ most visibly ($0.679$ against $0.477$); the edge removal
attains the higher operating point ($0.807$ against $0.779$, intervals
disjoint). Ranked formally, the requirement change wins; ranked operationally,
the edge removal does---and the requirement change buys its larger margin by
certifying a weaker property, which neither ranking states on its own. The rule
this pair supplies is general: across changed specifications a larger coordinate
means more defensible \emph{under the changed contract}, and does not on its own
establish a better design.

These two reversals show that structural and operational rankings need not
agree. Because all axes share one orientation convention, the fingerprint
identifies such reversals directly. Agreement supports the same ordering under
both assessments; inversion identifies configurations whose structural and
operational rankings differ.

\subsection{Attribution across Fingerprint Axes}
\label{sec:demo-attribution}

Both specification relaxations and both arena changes move the operating point,
but only the specification changes materially reshape shield latitude. Relative to
the reference, the arena changes (Cases~2 and~5) move $\mathrm{SLT}$ by $-1.6\%$
and $+2.0\%$ while moving $\mathrm{DDR}$ by $-57.9\%$ and $+49.7\%$; the
specification changes (Cases~3 and~4) move $\mathrm{SLT}$ by $-35.6\%$ and
$+13.9\%$, against $\mathrm{DDR}$ movements of $-11.9\%$ and $+44.5\%$. The two classes separate on the signed pair, not on
the outcome axis alone.

That separation locates where an intervention registers. Where the operating
point moves while $\mathrm{SLT}$ is close to unchanged, the defender's
edge-level restriction is not what shifted; where the two move together, it is.
What presents identically on a single axis, ``dominance moved,'' separates into
two response classes. The controlled one-component perturbation fixes which
component was changed, and the transition semantics of \Cref{tab:instance}
constrain how that change can propagate; the cross-axis response says where in
the measurement architecture it appeared.

$\mathrm{SLT}$ moves through two channels, and the two arena cases separate
them. It averages a per-state permitted-action fraction over the
shielded-reachable defender states, so it responds to a change in the masks, in
the domain, or both. Case~5 changes no mask at all: every reachable defender
position it retains carries the reference's permitted-action set, and the rise
in $\mathrm{SLT}$ comes entirely from eighty comparatively restricted positions
leaving the domain. Case~2 does narrow menus, at fifty-six positions, under a
smaller aggregate movement. A change in $\mathrm{SLT}$ therefore does not by
itself establish that the defender's menu narrowed.

The reading requires the signed response and not its magnitude. A ratio
$|\%\Delta\mathrm{DDR}/\%\Delta\mathrm{SLT}|$ discards direction, and grows
large precisely when the denominator is small: Cases~2 and~5 both produce a large
ratio while moving the operating point in opposite directions. The informative
object is the signed vector
$\Delta\mathbf{F} = (\Delta\mathrm{WIN}, \Delta\mathrm{STP},
\Delta\mathrm{SLT}, \Delta\mathrm{DDR})$, of which the pairing above is one
projection. Attribution of this kind is available whenever an analysis measures
the mechanism variables alongside the outcome variable, on commensurable
footing---here $\mathrm{SLT}$ measures the edge-level restriction the shield
imposes and $\mathrm{DDR}$ the outcome, and each alone discards what their joint
movement carries. The signed response vector identifies which fingerprint axes
change under each intervention.

\subsection{Objective-Specific Effects of an Adversary Capability}
\label{sec:demo-pricing}

Unlimited destruction turns out to be a formally severe and operationally minor
concession. Granting it costs $13.5\%$ of the winning region and $35.6\%$ of
shield latitude, but only $11.9\%$ of dominance. Unlimited destruction has a
substantially larger relative effect on shield latitude than on defender
dominance.

The transition semantics explain the split before any training run.
\textsc{Destroy} maps $C \to C$, $X \to Z$ and $D \to Z$ (\Cref{tab:instance}):
it never removes a host from the Clean set, and the hosts it does move were
already outside it, so no application changes $n_C$ and none changes $\mu$. The adversary's sole $\mu$-affecting action is \textsc{Spread}.
\textsc{Destroy} directly affects the availability requirement but not $\mu$;
it also reduces recovery from two defender actions to one. The model contains three coupled channels for the residual
$11.9\%$: the enlarged attractor narrows $\Safe(\cdot)$; \textsc{Destroy}
removes a host that could later have served as a \textsc{Spread} source; and the
shortened recovery path alters occupancy over a run.

Because the safety objective and dominance function are distinct, an adversary
capability can affect them by different amounts. Here the relative effect on
shield latitude is approximately three times the effect on defender dominance.

\subsection{Adaptive Effects of Adversary Restriction}
\label{sec:demo-trajectory}

With no destroy-budget constraint, Case~3 initially attains the highest defender
dominance of the five. In \Cref{fig:trajectories}, Case~3 rises above
$75\%$ clean within the first hundred episodes, then declines for the entire
remainder of training, crossing below the reference at episode $1{,}769$ of the
smoothed curves plotted there, and
settling at $0.475$ against the reference's $0.539$. By contrast,
\Cref{prop:refinement} guarantees that imposing $\phiA$ cannot shrink the
winning region and enlarges it here. The constraint helps
the defender formally and, for more than half of training, helps the adversary
operationally.

The budget constraint changes the distribution of exploratory actions. The
adversary's fifteen actions pair \textsc{Noop}, \textsc{Spread} and
\textsc{Destroy} with the five hosts, and \textsc{Spread} is the only action
class capable of changing $\mu$ (\Cref{sec:demo-pricing})---a particular
\textsc{Spread} is still a no-op wherever its target has no compromised or
detected in-neighbor. Unconstrained, the adversary draws from all fifteen, of
which the five \textsc{Spread} actions are a third; under $\phiA$, once the
budget is spent the five \textsc{Destroy} actions leave $\Legal(v)$ and hence
$\Safe(v)$, and \textsc{Spread} becomes five of the ten that remain. Exploration
is uniform over the admissible set, so \emph{conditional on an exploratory
move}---half of all moves at $\varepsilon = 0.5$---the constrained adversary
samples the one territory-changing class with probability $1/2$ rather than
$1/3$, a $50\%$ increase. Exploitative moves follow learned values rather than
this distribution. The unconstrained adversary initially selects more
$\mu$-neutral exploratory actions. As its value estimates converge, it selects
\textsc{Spread} more frequently and uses unrestricted \textsc{Destroy} actions
against $\phiD$, producing the decline in \Cref{fig:trajectories}.

The result shows that adversary restriction is monotone for winning-set
inclusion but not for finite-horizon learning dynamics. The trajectory is
therefore an auxiliary diagnostic beyond the settled fingerprint coordinate.

\subsection{The Fingerprint}
\label{sec:demo-fingerprint}

The fingerprints summarize the preceding comparisons on a shared absolute
scale (\Cref{fig:five_fingerprints}). Arena changes appear mainly on the
adaptive axis, specification changes affect all four axes, and order inversions
appear as opposing static and adaptive rankings.

\begin{figure}[ht]
\centering
\includegraphics[width=\columnwidth]{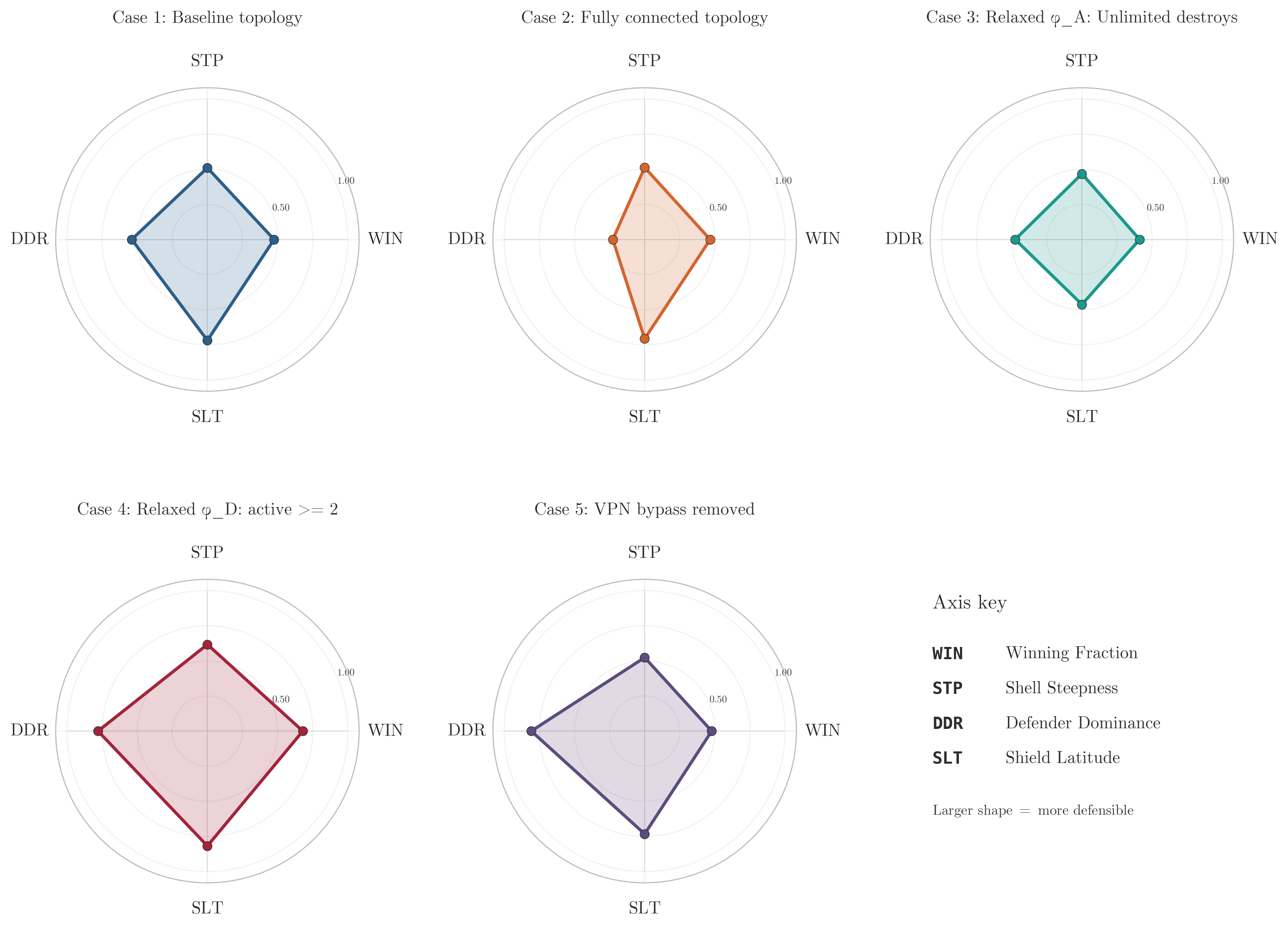}
\caption{Defensibility fingerprints of the five configurations on the absolute
scale of \Cref{sec:fingerprint}. The arena perturbations (Cases~2 and~5) differ
from the reference almost only on the adaptive axis; the specification
perturbations (Cases~3 and~4) reshape the static axes as well.}
\label{fig:five_fingerprints}
\Description{Five radar charts shown as small multiples, one per configuration, each using the same four defensibility axes and the same absolute zero-to-one scale. Placing the five polygons side by side shows how each arena or specification perturbation reshapes the defensibility profile relative to the reference; the two arena cases retain a reference-like shape on the three static axes while differing sharply on the defender-dominance axis, whereas the two specification cases change shape on all four.}
\end{figure}

Cases~2 and~5 each lie within $2.7\%$ of the reference on all three static axes
and differ by $58$ percentage points on the fourth.

Uniform dominance produces larger values on every axis, while the polygon
shape preserves the cross-axis differences reported above.

\subsection{The Instantiation Contract}
\label{sec:demo-boundary}

Three choices in the preceding evaluation belong to the instantiation rather
than to the framework. Together they fix what the instrument means here, and a
different system would fix them differently.

\paragraph{The dominance function.}
$\mu = n_C / |\calH|$ measures the share of hosts in a clean state.
$\mu$ is whatever quantity distinguishes a good position from a bad one within
the safe region, and $\phiD$ is what must never be violated; their separation
enables the objective-specific comparison in \Cref{sec:demo-pricing}.

\paragraph{The reward form.}
The zero-sum form $R_D = 2\mu - 1$ was adopted after per-status rewards---the
defender scoring per Clean host and the attacker per Compromised host---produced
a degenerate operating point under shielded action sets, both agents settling on a
locally favorable configuration and selecting \textsc{Noop} indefinitely. An
instantiation must supply an objective under which the two agents contest the
same quantity.

\paragraph{The engagement reset.}
When play reaches an all-Clean configuration, then with probability $p = 0.1$ the
attacker respawns at the gateway and both automata reset, beginning a new
engagement (\Cref{sec:respawn}); reaching all-Clean is the trigger, not the
result. What $\mathrm{DDR}$ is measured on is therefore a renewal process over
repeated legal engagement segments of the same certified arena, rather than a
single unbroken play. The protocol penalizes pre-emptive isolation because it maps
$C \to I$, removing a host from the clean count and the active count at once,
lowering $\mu$ while advancing $\AD^{(2)}$ toward violation. What an engagement
is, and what resets one, is a property of the system.

\subsection{Synthesis}
\label{sec:synthesis}

The evaluation shows that the certificate and four-axis fingerprint distinguish
configurations sharing the same verdict, localize structural effects by
attractor depth, identify order inversions, and separate objective-specific
responses. These results demonstrate comparative analysis from one encoded
system.


\section{Related Work}
\label{sec:related}

\paragraph{Shielded Reinforcement Learning.}
The foundational ShRL pipeline---LTL specification $\to$ DFA $\to$ product MDP $\to$ shield---was established by Bloem et al.~\cite{bloem2015} and Alshiekh et al.~\cite{alshiekh2018}, with extensions to $k$-stabilizing shields~\cite{konighofer2017} and probabilistic settings~\cite{jansen2020,hameldelecourt2025}, and a recent CACM survey~\cite{konighofer2025}. These works model the interaction as agent versus adversarial environment (nature). The environment is treated adversarially for the purpose of the guarantee---a safety game is solved against every environment behavior---but it is not ordinarily an independently specified, reward-seeking player whose own shield-bounded adaptation is part of what is analyzed, and the synthesis target is a runtime filter for the agent. Multi-agent shielding~\cite{elsayedaly2021,xiao2023} shields cooperative agents under a shared specification. POMDP-shielding extensions~\cite{carr2023,melcer2024} address partial observability under explicit abstraction assumptions. Across this literature, the winning region appears as an intermediate computation; the shield is the deliverable. The framework presented here inverts the role: the certificate is the deliverable, the shield is its witness, and two independent automata---one per specification---enter at different stages of the safety game's solution.

\paragraph{Reactive Synthesis with Environment Assumptions.}
Reactive synthesis distinguishes assumptions from guarantees and wires the
distinction into the synthesis objective: \emph{env satisfies $A$ $\to$ sys
satisfies $G$}~\cite{chatterjee2008,piterman2006,bloem2012}. Such an assumption
may itself be derived by pruning environment edges~\cite{chatterjee2008}.
Shielded analysis gives the distinction a different analytical architecture:
$\phiA$ is an externally declared threat boundary retained as an independently
manipulable design variable, while $\phiD$ defines the safety requirement.
Because they modify different stages of the fixed point, changes to requirement
and adversary capability can be certified, measured, and compared separately
(\Cref{sec:asym-general}, \Cref{sec:demo-inversion},
\Cref{sec:demo-pricing}, \Cref{sec:demo-trajectory}).

\paragraph{Game-Theoretic Security.}
Stackelberg security games~\cite{kiekintveld2009}, adversarial patrolling~\cite{klasna2021}, and dynamic resource allocation~\cite{guan2023} model strategic attacker-defender interaction using optimization-theoretic formulations: equilibrium computation under utility specifications, not winning-region computation under temporal-logic specifications. Shielded analysis contributes an automata-theoretic analysis of the same strategic domain, returning temporal-logic certificates, fixed-point structure, action latitude, and adaptive operating quality rather than equilibrium allocations.

\paragraph{Quantitative Games.}
Mean-payoff parity games and their relatives combine a qualitative
$\omega$-regular objective with a quantitative long-run payoff and compute a
value or optimal strategy for the combined objective~\cite{chatterjee2005}.
Shielded analysis deliberately preserves the certificate, winning-region
extent, forced-depth geometry, action latitude, adaptive operating point, and
learning trajectory as separate outputs. This separation makes disagreement,
rank inversion, objective-specific price, and set-versus-process
nonmonotonicity directly observable rather than collapsing them into a single
optimized quantity.

The network instantiation additionally connects the framework to three applied
research areas.

\paragraph{Network Configuration Verification.}
Static analysis tools decide properties of a network's forwarding behavior
directly from its configurations: reachability, isolation, and policy
compliance, checked proactively over all packet headers and, in the
symbolic formulations, over all environments a configuration may
face~\cite{fogel2015,beckett2017,headerspace}. These tools and the framework
here take the same artifact as input---the configuration a design review
already inspects---and answer different questions about it. Configuration
verification asks what the network does under a given environment; the
certificate asks whether a defender strategy exists that preserves a temporal
safety property against every adversary behavior a second specification
admits. Neither substitutes for the other: reachability is a property of the
static graph, defensibility a property of the game played on it.

\paragraph{Network Security Metrics.}
Attack-graph methods quantify vulnerability composition, probabilistic compromise likelihood, and hardening options over graph-based attack models, including cyclic AND/OR and Bayesian variants~\cite{wang2007,pamula2006,zenitani2023,wang2008}. They do not compute temporal-logic winning regions, attractor shells, or shield-induced action restrictions. The three static metrics of \Cref{sec:metrics} operate on dynamic safety-game attractor decompositions under worst-case adversarial interaction with formal temporal-logic specifications, and the adaptive metric reads shield-constrained play inside the region they describe. The distinction is structural: attack-graph metrics measure static vulnerability composition; the metrics here measure dynamic defensibility under strategic interaction. Shield latitude measures action restrictions that are not represented by the attack-graph composition or empirical agent-performance metrics reviewed here.

\paragraph{RL for Cybersecurity.}
The CybORG/CAGE Challenge~\cite{cyborg,standen2021,kiely2025} has run multiple iterations with extensive MARL submissions for autonomous network defense. Safety in this line of work is enforced through soft methods: reward shaping, constrained optimization, and Lagrangian penalties. Existing CAGE-style evaluations report empirical agent performance against a fixed adversary distribution; they do not produce topology-level temporal-logic defensibility certificates or shield-derived winning-region diagnostics. Empirical success against a particular adversary distribution cannot prove the absence of a winning attacker strategy. Shielded analysis supplies that topology-level formal certificate and its structural diagnostics while retaining adaptive evaluation on the same encoded system.


\section{Discussion}
\label{sec:discussion}

\paragraph{An instrument for relational questions.}
Shielded analysis computes a certificate, structural decomposition, shield, and
operating point from one encoded system. This common basis permits direct
comparison of the individual readings and their relationships.

\paragraph{What the two layers respond to.}
The controlled evaluation shows that the formal and operational layers respond
differently to system changes. Both arena
perturbations of \Cref{sec:experiments}---adding edges, and removing the
bypass---preserve which hosts the adversary can eventually reach from the entry
point; they change only the length and multiplicity of the paths by which reach
is achieved. The certificate records whether a safe defender strategy exists, not
how long the adversary's forcing sequences are or how many paths realize them, and
it is unchanged across both. The static readings are not invariant but
comparatively insensitive: they stay within $2.7\%$ of the reference while
dominance moves by more than half (\Cref{sec:demo-operation}), and the shell
decomposition registers the bypass deletion at ranks~3 and~4
(\Cref{sec:demo-structure}). Over a
topology-only corpus of $90$ randomly generated arenas held to the reference
specifications---every member defensible, and routinely disconnecting a host and
so altering reachability---the static axes respond strongly to edge density:
tie-corrected Spearman $-0.852$ for $\mathrm{WIN}$ and $-0.868$ for
$\mathrm{SLT}$. Across the controlled cases and topology corpus,
reachability-changing interventions primarily affect winning-region extent and
shield latitude, whereas reachability-preserving path changes appear in
attractor depth and adaptive operating quality.

An architect comparing two segmentation designs that differ in redundancy but
not in ultimate reachability should expect the certificate and the static
profile to agree, and should read the difference from the operational layer. One
isolating an asset, or weighing two candidate safety requirements, should expect
the formal layer to separate the designs.

\paragraph{Where the abstraction applies.}
The requirements are structural, not domain-specific: a system whose evolution
is a turn-based interaction between a controller and a strategic adversary, a
safety property the controller must preserve, a characterization of which
adversary behaviors are in scope, and a dominance quantity distinguishing
strong positions from weak ones within the safe region. Wherever these co-occur
and the analysis is offline, the same reading is available---robotics near
hazardous equipment, industrial control under operator error or sabotage,
mission planning against a reactive opponent, cyber-physical systems more
broadly. Power-system contingency analysis is a case in point: the $N\!-\!k$
criterion~\cite{bienstock2010} already fixes which simultaneous component
failures are in scope, which
is an adversary specification in the sense used here. Instantiating any of them
is the modeling exercise of \Cref{sec:demo-boundary}: supplying those four
inputs. The network study carries that abstraction through one complete
instantiation.

The framework applies to deterministic, fully observable, turn-based systems
with a safety objective, an explicit threat boundary, and a dominance function.
Extensions to stochastic, partially observable, or simultaneous-action settings
require corresponding formal solvers.

\paragraph{Computational scope.}
The explicit product grows as $|\St|^{|\calH|}$~\cite{baier2008}: the five-host
segment studied here yields $150{,}000$ positions. At this scale, attractor
computation is effectively instantaneous. The ten adaptive runs account for
nearly all of the approximately ten CPU minutes required per configuration
(\Cref{app:environment}).

\section{Conclusion}
\label{sec:conclusion}

This paper introduced \emph{shielded analysis}, a design-time framework for
certifying and characterizing defensibility under independently specified
safety requirements and threat boundaries. Its asymmetrically constrained
safety game assigns the two specifications to distinct stages of fixed-point
solution: the defender specification seeds the unsafe set, while the adversary
specification filters adversary moves during attractor computation. The
resulting winning region witnesses the defensibility certificate and yields its
executable shield.

The certificate answers the existence question; the four-coordinate
defensibility fingerprint characterizes winning-region extent, forced-depth
geometry, action latitude, and adaptive operating quality. Computing these
readings on one encoded system enables direct analysis of their agreement,
inversion, and objective-specific responses.

The controlled evaluation exposed properties that neither formal verification
nor adaptive evaluation reveals alone. Across a reference segment and four
perturbations, configurations indistinguishable to certification diverged
operationally---one forgotten bypass rule moved defender dominance from
$53.9\%$ to $80.7\%$---operational shifts were located on the mechanism axis
that moved with them, and a capability costing $35.6\%$ of shield latitude cost
only $11.9\%$ of dominance, an approximately threefold difference in relative
response between the two axes. The network evaluation establishes these
analytical capabilities through controlled interventions.

The framework applies when the safety requirement and admissible adversary
behavior can be specified independently. The network instantiation demonstrates
its use for comparing architecture, requirement, and capability changes before
deployment.


\bibliography{bastion}

\appendix

\section{Execution Environment}
\label{app:environment}

All computation was performed on a single consumer laptop; the workload
(explicit-state attractor computation and tabular minimax-Q) is CPU-bound and
single-machine, and no GPU acceleration was used. \Cref{tab:environment}
records the hardware and the principal software versions.

\begin{table}[h]
\centering
\caption{Execution environment. All experiments were run on a single machine;
computation is CPU-only.}
\label{tab:environment}
\smallskip
\begin{tabular}{@{}ll@{}}
\toprule
Component & Specification \\
\midrule
\multicolumn{2}{@{}l}{\textit{Hardware}}\\
CPU        & AMD Ryzen~7 8845HS, 8 cores / 16 threads, 3.80\,GHz \\
Memory     & 16\,GB RAM (14.9\,GiB usable) \\
GPU        & Present (NVIDIA RTX~4060 Mobile) but unused; computation is CPU-only \\
\addlinespace
\multicolumn{2}{@{}l}{\textit{Software}}\\
OS         & Fedora Linux~44 (Workstation), kernel 7.0.12 (x86\_64) \\
Python     & 3.12.13 \\
NumPy      & 2.3.5 \\
SciPy      & 1.16.3 \\
Matplotlib & 3.10.8 \\
\bottomrule
\end{tabular}
\end{table}

\noindent
On this configuration, fixed-point attractor computation is effectively
instantaneous, successor-graph construction takes approximately five seconds per
case, and one MARL training run (3{,}000 episodes $\times$ 1{,}000 steps)
completes in roughly 50 seconds; total wall time across all five cases at ten
seeds each is under one hour.

\end{document}